%% file: iclr2027_conference.tex
\documentclass{article} 
\usepackage{iclr2027_conference,times}

\input{math_commands.tex}

\usepackage[utf8]{inputenc} 
\usepackage[T1]{fontenc}    
\usepackage{hyperref}       
\usepackage{url}            
\usepackage{booktabs}       
\usepackage{amsmath}
\usepackage{amssymb}
\usepackage{amsfonts}       
\usepackage{amsthm}
\usepackage{graphicx}
\usepackage{wrapfig}
\usepackage{caption}
\usepackage{nicefrac}       
\usepackage{microtype}      
\usepackage[table]{xcolor}  
\usepackage{pifont}
\usepackage{multirow}

\colorlet{bestred}{red!6}
\colorlet{secondorange}{orange!8}
\colorlet{thirdblue}{cyan!6}
\definecolor{crossrefred}{HTML}{CF553D}

\hypersetup{
    colorlinks=true,
    linkcolor=crossrefred,
    filecolor=magenta,
    urlcolor=cyan,
    citecolor=cyan,
}

\newcommand{\bestcell}[1]{\cellcolor{bestred}{#1}}
\newcommand{\secondcell}[1]{\cellcolor{secondorange}#1}
\newcommand{\thirdcell}[1]{\cellcolor{thirdblue}#1}

\newtheorem{lemma}{Lemma}

\title{OceanFlow: Coherent Flow Matching for\\Probabilistic Ocean Modeling}

\author{Yuan Gao$^{1}$\thanks{Equal contribution. $^{ \dagger}$Corresponding author.}   , Fan Xu$^{2 *}$, Hao Wu$^{1 \dagger}$, Xiaomeng Huang$^{1 \dagger}$ \\
$^1$Tsinghua University, $^2$University of Science and Technology of China
}

\iclrfinalcopy 
\begin{document}

\newcommand{\model}{{\fontfamily{lmtt}\selectfont \textbf{OceanFlow}}}
\newcommand{\method}{\model}

\maketitle

\input{sec/0_abstract}    
\input{sec/1_intro}
\input{sec/3_Method}
\input{sec/4_Experiments}
\input{sec/5_conclusions}

\clearpage

\subsection*{AI use statement}

We use AI to check for grammatical errors in the text; the AI is not involved in ideas, experimental design, or the actual writing of the paper.

\subsection*{Ethics statement}

This work complies with the ICLR Code of Ethics. All datasets and pretrained models are publicly available and used under their respective licenses. We have carefully designed our experiments to minimize risks of bias or discrimination. We uphold the principles of transparency and integrity throughout the work.

\subsection*{Reproducibility statement}

To support reproducibility, we document data sources, preprocessing, and splits in Appendix~\ref{appendix:data}, and provide architectural details, training objectives, and inference procedures in Appendix~\ref{app:detailed_method}. Evaluation metrics and the theoretical proof are presented in Appendices~\ref{app:metrics} and~\ref{app:coherence_proof}, respectively. Training settings and hardware are reported in the main text.

\subsubsection*{Acknowledgments}
This work was supported by the National Natural Science Foundation of China (42125503, 42430602).

\bibliography{iclr2027_conference}
\bibliographystyle{iclr2027_conference}

\appendix
\clearpage
\input{sec/X_suppl}

\end{document}

%% file: math_commands.tex
\usepackage{amsmath,amsfonts,bm}

\def\eqref#1{equation~\ref{#1}}

\def\1{\bm{1}}

\DeclareMathAlphabet{\mathsfit}{\encodingdefault}{\sfdefault}{m}{sl}
\SetMathAlphabet{\mathsfit}{bold}{\encodingdefault}{\sfdefault}{bx}{n}



%% file: sec/0_abstract.tex
\begin{abstract}
Probabilistic ocean modeling is essential for risk-aware marine operations, climate analysis, and scientific decision-making, where downstream decisions depend on the conditional distribution of future ocean states rather than a single deterministic trajectory. Recent data-driven ocean models have made deterministic global ocean modeling increasingly accurate and efficient, but probabilistic ocean modeling remains underdeveloped. Specifically, perturbation-driven ensembles can create spread without explicitly optimizing the joint distribution of their members, often yielding poor calibration and physical inconsistency. To address these issues, we propose \model{}, a coherent flow matching framework for probabilistic global ocean modeling. \model{} first establishes a sharp, stable mean anchor for the conditional ensemble. Flow matching then learns a coherent state-conditioned pathwise law that jointly reshapes the residual distribution around this anchor at each transition, preserving the predictive center while refining state-dependent uncertainty. We further introduce a Factorized Dynamics Operator (FDO), which combines factorized transitions with a learned state-conditioned response to parameterize the mean anchor and residual-flow velocity. Extensive experiments show that \model{} improves ensemble-mean accuracy, probabilistic reliability, and physical-consistency diagnostics over competitive baselines, supporting coherent flow matching as an effective route toward data-driven probabilistic ocean modeling. Code is available at \url{https://github.com/YuanGao-YG/OceanFlow}.

\end{abstract}

%% file: sec/1_intro.tex
\section{Introduction}

Ocean modeling is essential for risk-aware marine operations, climate analysis, and scientific decision-making~\citep{kalnay2003atmospheric,lellouche2018recent,tan2025leveraging}. Incomplete observations leave uncertainty in the initial ocean state, while nonlinear dynamics can amplify small state errors into substantially different future evolutions~\citep{lorenz2017deterministic,leith1974theoretical}. Even an accurate deterministic estimate cannot characterize the range of plausible outcomes or their relative likelihoods. Probabilistic modeling can address this limitation by representing the conditional distribution of future ocean states, providing uncertainty information beyond a single trajectory. Recent data-driven models, including AI-GOMS~\citep{xiong2023ai}, XiHe~\citep{wang2024xihe}, LangYa~\citep{yang2024langya}, and WenHai~\citep{cui2025forecasting}, have substantially advanced deterministic global ocean modeling, combining competitive predictive accuracy with considerably faster inference than traditional numerical models. These gains make ocean modeling more affordable and provide a strong foundation for efficient ensemble generation. However, despite this progress in deterministic modeling, data-driven probabilistic ocean modeling remains comparatively underexplored.

Similar to initial-condition perturbation in numerical ensemble forecasting, data-driven ocean models can be extended to probabilistic modeling by perturbing their initial states and evolving the resulting members with a pretrained deterministic model~\citep{hao2025deep}. Prescribed perturbations can generate ensemble spread and improve ensemble-mean accuracy through averaging. However, members evolve independently without joint distributional optimization, potentially yielding distributions that deviate from physically plausible ocean variability and fail to faithfully represent uncertainty in the underlying dynamics. FuXi-ONS~\citep{huang2026data} advances beyond prescribed noise by learning state-dependent, spatially structured perturbations for autoregressive ocean prediction. This provides a more adaptive representation of ensemble variability. Nevertheless, its memberwise prediction objective leaves ensemble calibration implicit, while member updates do not explicitly account for the evolving ensemble context. These advances motivate a central question: \textbf{\textit{how can the ensemble be jointly refined to more faithfully characterize uncertainty in the underlying ocean dynamics while maintaining a controlled predictive center?}} 

\begingroup
\setlength{\intextsep}{5pt}
\setlength{\columnsep}{10pt}
\begin{wrapfigure}{r}{0.55\textwidth}
    \vspace{-5pt}
    \centering
    \includegraphics[width=\linewidth]{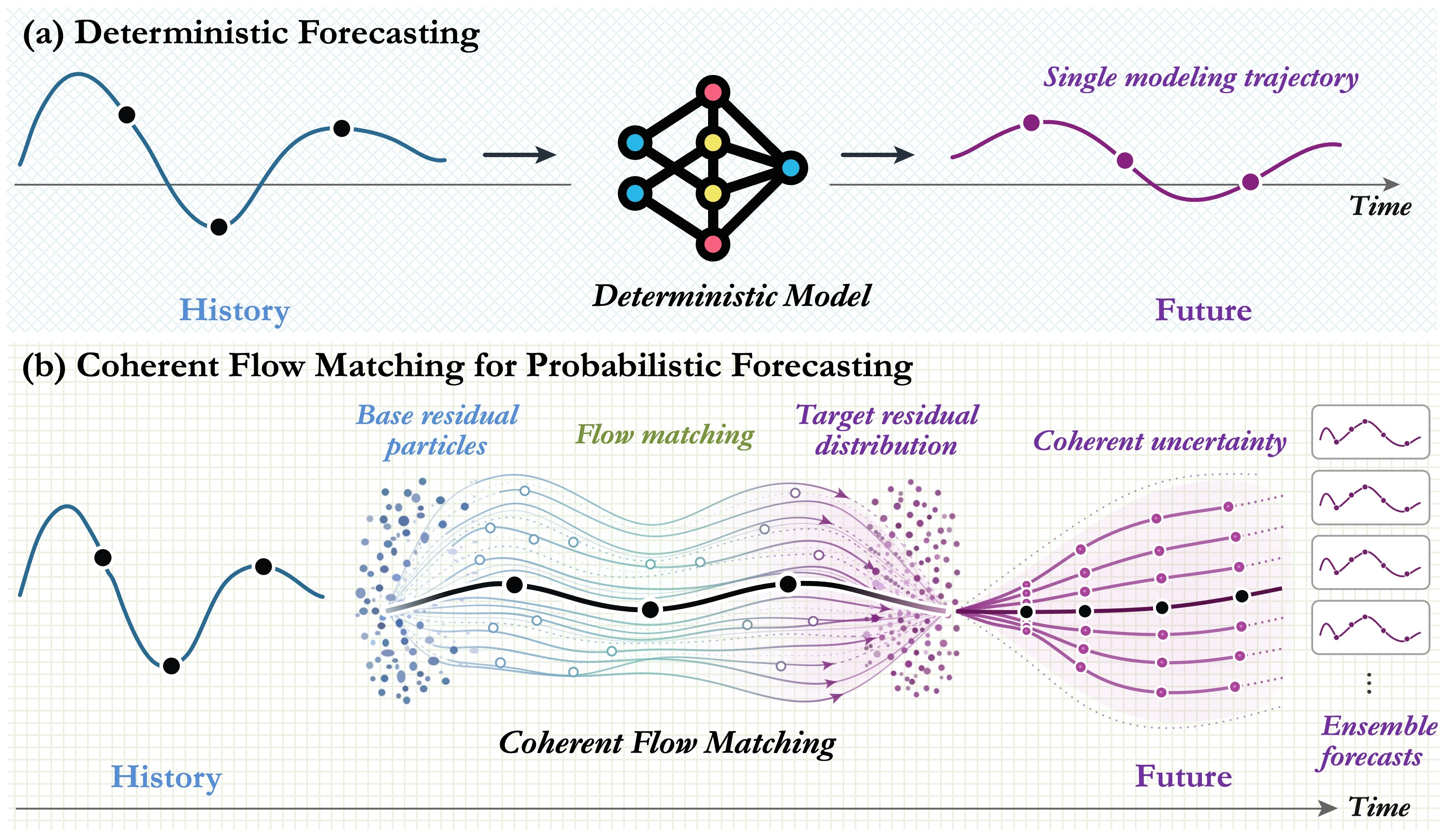}
    \setlength{\abovecaptionskip}{4pt}
    \setlength{\belowcaptionskip}{0pt}
    \vspace{-15pt}
    \caption{Learning coherent ensemble laws for ocean modeling. (a) Deterministic modeling produces a single trajectory. (b) \model{} represents ocean uncertainty through a mean-anchored ensemble, jointly refining member residuals under a coherent state-conditioned pathwise law at each transition.}
    \vspace{-5pt}
    \label{fig:intro}
\end{wrapfigure}

To answer this question, we propose \model{}, a coherent flow matching framework for probabilistic ocean modeling. As illustrated in Figure~\ref{fig:intro}, \model{} learns a coherent state-conditioned pathwise law for the residual ensemble around a controlled mean anchor at each transition. This law couples member-specific transport paths through a shared velocity field conditioned on the ocean state and evolving ensemble context. Ensemble diversity is thus jointly refined as a conditional residual distribution, with final recentering preserving the predictive center. We further introduce a Factorized Dynamics Operator (FDO) that parameterizes both the mean anchor and residual-flow velocity through a unified composition of factorized dynamics and learned state-conditioned response. Our contributions can be summarized as:
\par
\endgroup

\begin{itemize}
    \item We propose \model{}, a coherent flow matching framework for probabilistic ocean modeling that learns a coherent state-conditioned pathwise law around the mean anchor center. By coupling member residuals through the evolving ensemble context, it jointly refines the conditional distribution while preserving a controlled mean anchor.
    \item We introduce a Factorized Dynamics Operator (FDO), which unifies factorized dynamics and a learned state-conditioned response within a single conditional operator formulation. FDO supports both an accurate deterministic reference for the mean anchor and a structured velocity field for residual transport.
    \item Extensive experiments on global ocean simulation and forecasting tasks demonstrate that the proposed framework achieves improved ensemble-mean accuracy, probabilistic reliability, and physical consistency.
\end{itemize}

%% file: sec/3_Method.tex
\section{Method}

\subsection{Problem Formulation}
Let $\mathbf{O}_t\in\mathcal{X}$ denote the ocean state and $\mathbf{F}_t,\mathbf{F}_{t+1}$ the external forcing. Simulation uses observed future forcing, whereas prediction uses forcing generated by a separately trained OneForecast atmospheric model~\citep{gao2025oneforecast}. Both settings follow the same conditional ocean transition. Given $c_t=(\mathbf{O}_t,\mathbf{F}_t,\mathbf{F}_{t+1})$, probabilistic modeling seeks an $M$-member conditional ensemble,
\begin{equation}
\widehat{\mathbf{O}}_{t+1}^{1:M}
\sim
Q_\theta^M(\cdot\mid c_t),
\qquad
Q_\theta^M(\cdot\mid c_t)\in\mathcal{P}(\mathcal{X}^{M}).
\label{eq:ensemble_distribution}
\end{equation}
As depicted in Figure~\ref{fig:main}, \model{} constructs this distribution through a mean anchor and a coherent state-conditioned pathwise law for member anomalies. The anchor controls the predictive center, while ensemble-conditioned residual transport learns the distribution around it. A Factorized Dynamics Operator (FDO) provides a unified parameterization for the deterministic reference underlying the anchor and the residual-flow velocity.

\begin{figure}[t]
    \centering
    \includegraphics[width=0.96\textwidth]{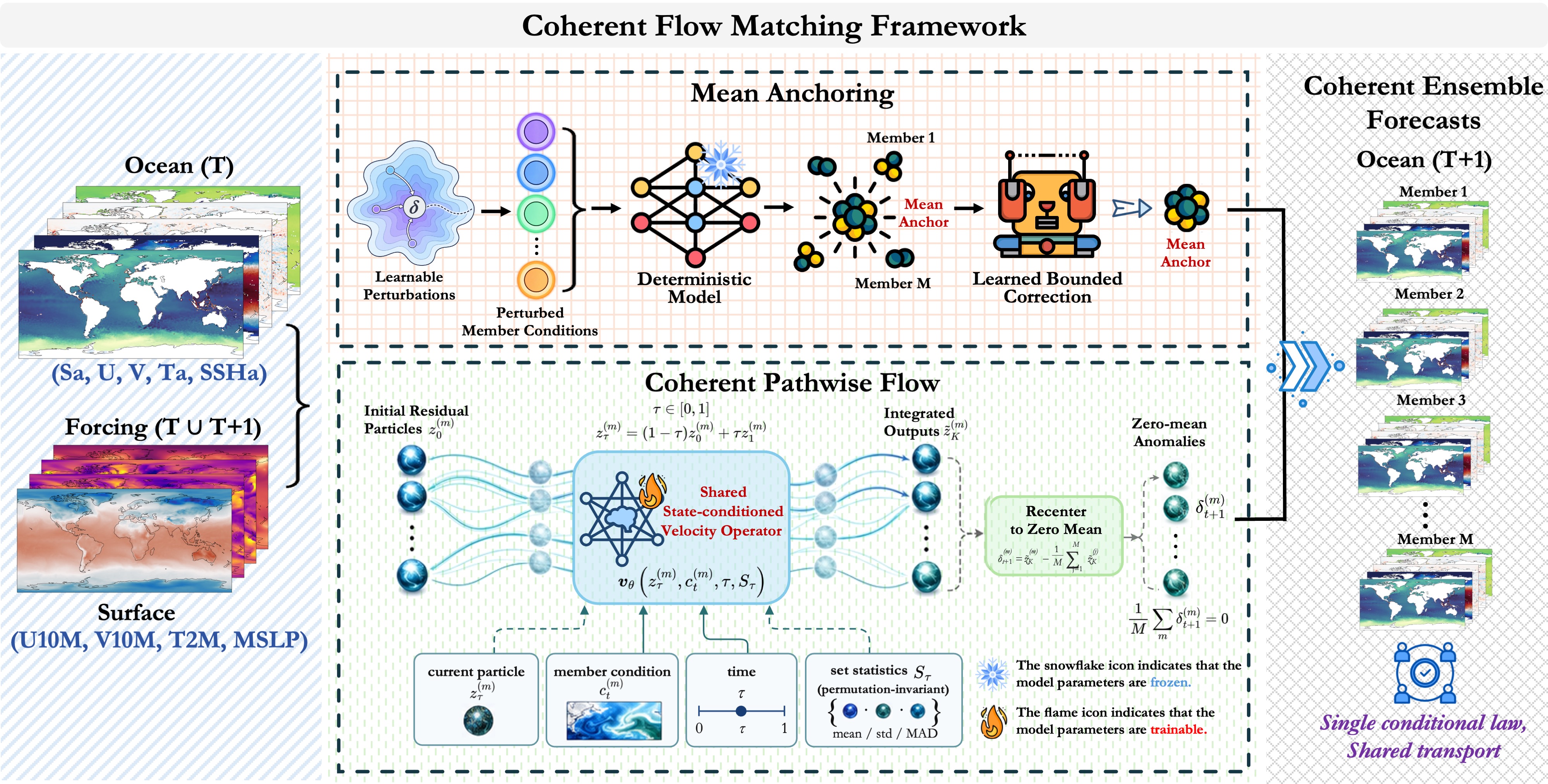}
    \vspace{-5pt}
    \caption{\textbf{Overview of \model{}.} The mean anchor controls the predictive center, while coherent state-conditioned pathwise flow jointly transports member residuals. Learned residuals refine the reference ensemble, while final recentering preserves the mean anchor. FDO parameterizes both the deterministic reference and residual velocity through a unified composition of factorized dynamics and learned state-conditioned response.}
    \label{fig:main}
    \vspace{-15pt}
\end{figure}

\subsection{Coherent Flow Matching}

Coherent flow matching couples member residual paths through ensemble-conditioned dynamics while retaining explicit control of their common center. The law acts at each autoregressive transition in residual flow time $\tau$: a deterministic reference advances the member conditions, and the learned transport determines the evaluated conditional ensemble.

\paragraph{Mean anchoring.}
Mean anchoring separates central-state accuracy from residual uncertainty. A pretrained deterministic FDO instance $f_{\mathrm{det}}$ advances member conditions $c_t^{(m)}$, formed by perturbing the ocean state while sharing the same forcing. The reference ensemble can be expressed as
\begin{equation}
    x_{\mathrm{base},t+1}^{(m)} = f_{\mathrm{det}}(c_t^{(m)}),
    \qquad
    \mu^{\mathrm{base}}_{t+1}
    =
    \frac{1}{M_b}\sum_{m=1}^{M_b}x_{\mathrm{base},t+1}^{(m)} ,
    \label{eq:base_anchor}
\end{equation}
where $M_b$ is the reference-ensemble size used to estimate $\mu^{\mathrm{base}}_{t+1}$, while $M$ is the number of residual particles and corresponding output members; $M_b=M$ at evaluation. A learned correction adapts this reference center,
\begin{equation}
    \mu_{t+1}
    =
    \mu^{\mathrm{base}}_{t+1}
    +
    \mathcal{P}_o\!\left[
    \alpha_{\mu}
    \tanh\!\left(
        \frac{\mathcal{P}_o g_{\psi}(c_t,\mu^{\mathrm{base}}_{t+1})}{\alpha_{\mu}}
    \right)\right],
    \label{eq:mean_correction}
\end{equation}
where $\alpha_\mu$ controls the correction scale and $\mathcal{P}_o$ removes each channel's spatial mean over valid ocean points. The resulting ensemble separates the mean anchor from zero-mean member anomalies,
\begin{equation}
    \hat{x}_{t+1}^{(m)} = \mu_{t+1} + \delta_{t+1}^{(m)},
    \qquad
    \frac{1}{M}\sum_{m=1}^{M}\delta_{t+1}^{(m)} = 0.
    \label{eq:forecast_decomp}
\end{equation}
This decomposition allows the mean branch to learn a central correction without letting residual transport displace the anchor.

\paragraph{Coherent state-conditioned pathwise flow.}
The residual ensemble evolves under one velocity field whose dependence on symmetric ensemble summaries couples the member paths. Let $Z_0=\{z_0^{(m)}\}_{m=1}^{M}$ be smooth normalized residual noise. The training target $Z_1$ combines a spread target derived from the verifying center error with zero-mean patterns of that error; its construction is given in Appendix~\ref{app:flow_matching_details}. For a shared interpolation time $\tau\sim\mathcal{U}(0,1)$,
\begin{equation}
    z_{\tau}^{(m)}
    =
    (1-\tau)z_0^{(m)} + \tau z_1^{(m)}, \qquad
    u_{\tau}^{(m)} = z_1^{(m)} - z_0^{(m)}.
\end{equation}
\model{} learns a shared residual velocity field $v_\theta$ by the flow matching objective
\begin{equation}
    \mathcal{L}_{\mathrm{FM}}
    =
    \mathbb{E}_{(c_t,y),\Xi,\tau}
    \left[
    \frac{1}{M}\sum_{m=1}^{M}
    \left\|
    v_{\theta}\!\left(
        z_{\tau}^{(m)}, c_t^{(m)}, \tau, S_{\tau}
    \right)
    -
    u_{\tau}^{(m)}
    \right\|_{2,o}^2
    \right],
    \label{eq:fm}
\end{equation}
where $y$ is the verifying state, $\Xi$ collects the stochastic member inputs, and $S_\tau=S(Z_\tau)$ contains the ensemble mean, standard deviation, and mean absolute deviation. These summaries are recomputed from the joint residual state, making each member's velocity responsive to the ensemble context rather than to its own condition alone. Norms use the masked-grid convention in Appendix~\ref{app:training_objectives}.

Inference advances the residual ensemble with the same shared velocity,
\begin{equation}
    z_{k+1}^{(m)}
    =
    \mathcal{P}_r\!\left[
    z_{k}^{(m)}
    +
    \Delta \tau\,
    v_{\theta}\!\left(
        z_k^{(m)}, c_t^{(m)}, \tau_k, S_k
    \right)\right],
    \qquad
    \tau_k = \left(k+\tfrac12\right)\Delta\tau,
    \quad \Delta\tau=N_\tau^{-1}.
    \label{eq:euler_transport}
\end{equation}
Here $\mathcal{P}_r$ applies the ocean mask and residual bound. The update defines a coherent member-set transport $\mathcal{T}_{\theta}^{M}$; shared parameters and symmetric context make it equivariant to jointly reordering particles and their conditions, as shown in Appendix~\ref{app:set_statistics}. Let $\tilde z_{N_\tau}^{(m)}$ denote the anomalies obtained by composing the reference-member anomalies with the transported residuals and applying spread calibration and stabilization. The final projection is
\begin{equation}
    \delta_{t+1}^{(m)}
    =
    \tilde{z}_{N_\tau}^{(m)}
    -
    \frac{1}{M}\sum_{j=1}^{M}\tilde{z}_{N_\tau}^{(j)},
    \label{eq:recenter}
\end{equation}
which fixes the empirical center at $\mu_{t+1}$ for every generated ensemble. Coherence thus resides in the jointly conditioned residual transport, with mean anchoring enforcing the center of the resulting law. Reference composition and calibration are specified in Appendix~\ref{app:flow_matching_details}.

\subsection{Factorized Dynamics Operator}

The Factorized Dynamics Operator (FDO) is a general-purpose conditional transition operator for mean-anchor and residual-velocity parameterization. It expresses each transition through factorized dynamics composed with a state-conditioned response, giving both instances a unified operator structure.

For operator input $\mathbf{X}$ and context $\mathbf{C}$, a conditional patch encoder establishes the state on which factorized dynamics act,
\begin{equation}
H_0
=
\mathcal{E}_{\theta}^{\mathrm{patch}}
\left(
\mathbf{X};\mathbf{C}
\right).
\label{eq:fdo_encoding}
\end{equation}

The transition is accumulated through interacting dynamical factors,
\begin{equation}
H_{\ell+1}
=
H_{\ell}
+
\Phi_{\ell}
\left(
H_\ell,\,
\mathcal{A}_{\ell}
\left(
\left\{
\mathcal{T}_{\ell,k}
\left(
\mathcal{P}_{\ell,k}(H_{\ell};\mathbf{C})
\right)
\right\}_{k=1}^{K}
;\,
H_{\ell},\mathbf{C}
\right)
;\mathbf{C}\right),
\qquad
\ell=0,\ldots,L-1 .
\label{eq:fdo_factorized_transition}
\end{equation}
where $\mathcal{P}_{\ell,k}$ induces factor $k$, $\mathcal{T}_{\ell,k}$ evolves it, and $\mathcal{A}_{\ell}$ combines cross-factor interactions with state-dependent selection. The update map $\Phi_\ell$ integrates this collective response into the patch state.

FDO maps this evolution to the target space through a state-conditioned response composition,
\begin{equation}
\widetilde{\mathbf{Y}}
=
\mathcal{R}_{\theta}(H_L),
\qquad
\mathcal{D}_{\theta}(\mathbf{X};\mathbf{C})
=
\left[
\mathcal{I}
+
\gamma\mathcal{S}_{\theta}(\,\cdot\mid\mathbf{X})
\right]
\left(
\widetilde{\mathbf{Y}}
\right)
=
\widetilde{\mathbf{Y}}
+
\gamma
\mathcal{S}_{\theta}
\left(
\widetilde{\mathbf{Y}},\mathbf{X}
\right).
\label{eq:fdo_closure}
\end{equation}
Here $\mathcal{R}_\theta$ restores the output field, while $\mathcal{S}_\theta$ conditions its response on both the decoded dynamics and the operator input, with learned coefficient $\gamma$. The mean and velocity instances use separate parameters and task-specific inputs; their shared composition is detailed in Appendix~\ref{app:velocity_operator}.

\subsection{Training Losses and Calibration}
\label{sec:training_losses}

Training first establishes a deterministic reference and then learns the coherent conditional ensemble around it. For normalized, masked ocean fields, the deterministic FDO instance $f_{\mathrm{det},\vartheta}$ minimizes the channel-averaged relative $L_2$ risk,
\begin{equation}
    \mathcal{L}_{\mathrm{det}}(\vartheta)
    =
    \mathbb{E}_{(c_t,y)\sim\mathcal{D}_{\mathrm{train}}}
    \left[
    \frac{1}{C_o}\sum_{q=1}^{C_o}
    \frac{\|M_{o,q}\odot(f_{\mathrm{det},\vartheta}(c_t)_q-y_q)\|_2}
         {\|M_{o,q}\odot y_q\|_2}
    \right],
    \label{eq:deterministic_pretraining}
\end{equation}
where $C_o$ is the number of ocean channels and $M_{o,q}$ the valid-ocean mask for channel $q$. The pretrained reference is then frozen. With $\widehat X=\{\hat x^{(m)}\}_{m=1}^{M}$, the probabilistic stage optimizes the joint mean--residual objective
\begin{equation}
    \mathcal{L}(\theta)
    =\mathbb{E}\!\left[
    \lambda_{\mathrm{fm}}\ell_{\mathrm{fm}}
    +\lambda_{\mathrm{mean}}\ell_{\mathrm{mean}}
    +\lambda_{\mathrm{prob}}\bigl(\mathsf{S}_M(\widehat X,y)+\eta\ell_{\mathrm{gain}}\bigr)
    +\lambda_{\mathrm{reg}}\ell_{\mathrm{reg}}
    \right].
    \label{eq:compact_training_objective}
\end{equation}
The expectation averages training transitions, stochastic member inputs, and flow time. Here $\ell_{\mathrm{fm}}$ is the sample-level integrand of Eq.~\eqref{eq:fm}; $\ell_{\mathrm{mean}}$ supervises the center and its improvement over the reference; and $\ell_{\mathrm{reg}}$ controls member error and spread calibration. Distributional supervision uses empirical ensemble CRPS,
\begin{equation}
    \mathsf S_M(\widehat X,y)
    =
    \frac{1}{M}\sum_m\|\hat x^{(m)}-y\|_{1,o}
    -\frac{1}{2M^2}\sum_{m,n}\|\hat x^{(m)}-\hat x^{(n)}\|_{1,o},
    \label{eq:crps}
\end{equation}
where $\|\cdot\|_{1,o}$ denotes the masked-grid absolute average. The score balances member--target agreement with ensemble dispersion, while $\ell_{\mathrm{gain}}$ penalizes CRPS exceeding that of the reference ensemble. Final recentering enforces the mean identity independently of optimization. Component definitions are given in Appendix~\ref{app:training_objectives}.

\subsection{Theoretical Analysis}

\begingroup
\itshape
For transition context $c$ and member conditions $C=(c^{(1)},\ldots,c^{(M)})$, the continuous-time idealization of coherent transport evolves the residual ensemble through
\begin{equation}
    \left[V_{\theta}^{M}(Z,C,\tau)\right]^{(m)}
    =
    v_{\theta}\!\left(
        z^{(m)},c^{(m)},\tau,S(Z)
    \right),
    \qquad
    \frac{\mathrm{d}Z_\tau}{\mathrm{d}\tau}
    =
    V_{\theta}^{M}(Z_\tau,C,\tau),
    \label{eq:joint_ensemble_flow}
\end{equation}
Let $\mathcal{T}_{\theta,C}^{M}$ denote either the time-one flow or the implemented finite-step transport, $\mathcal{K}_{\theta,c}^{M}$ the reference composition and calibration, and $\mathcal{R}^{M}$ the zero-mean projection. For stochastic input $\Xi=(Z_0,C)\sim\rho_0^M(\cdot\mid c)$, the resulting conditional ensemble law is
\begin{equation}
    Q_{\theta}^{M}(\cdot\mid c)
    =
    \left(G_{\theta,c}^{M}\right)_{\#}
    \rho_0^{M}(\cdot\mid c),
    \qquad
    G_{\theta,c}^{M}(\Xi)
    =
    \mathbf{1}_{M}\otimes\mu_\theta(c,\Xi)
    +
    \mathcal{R}^{M}\!\left(
        \mathcal{K}_{\theta,c}^{M}\!\left(
        \mathcal{T}_{\theta,C}^{M}(Z_0),\Xi\right)
    \right).
    \label{eq:coherent_conditional_pushforward}
\end{equation}
\begin{lemma}[Coherent conditional law]
\label{lem:coherent_transport}
Under the assumptions in Appendix~\ref{app:coherence_proof}, each ensemble-conditioned transport induces its own conditional law $Q_\theta^M$ with empirical mean $\mu_\theta(c,\Xi)$ for every stochastic input.
\end{lemma}
\endgroup

%% file: sec/4_Experiments.tex
\section{Experiments}
We evaluate \model{} on ocean simulation and forecasting tasks in terms of ensemble-mean accuracy, probabilistic reliability, ocean-specific diagnostics, physical consistency, and model complexity.

\subsection{Experimental Setup}

\noindent\textbf{\ding{182} Dataset}~ We evaluate \model{} on GLORYS12 reanalysis with 93 ocean variables: salinity, temperature, and two velocity components at 23 depths, plus sea surface height. Four ERA5 forcing variables at two consecutive time steps provide eight atmospheric input channels. Data are processed at daily temporal and 1° spatial resolution, using daily-mean ocean states and forcings sampled at 12:00 UTC. We use 1993--2017 for training, 2018--2019 for validation, and 2020 for testing, with additional generalization tests on 2024 data. Further data details refer to the Appendix.

\noindent\textbf{\ding{183} Evaluation Metrics}~ We evaluate ensemble-mean skill with RMSE and ACC, and probabilistic quality with empirical ensemble CRPS, spread--skill ratio (SSR), rank histograms, and coverage calibration curves. RankDev and CalibErr quantify deviations from rank uniformity and nominal coverage, respectively. Metric definitions appear in Appendix~\ref{app:metrics}.

\noindent\textbf{\ding{184} Baselines}~ We compare \model{} with a diverse set of baselines: the general spatiotemporal model \textit{\textbf{SimVP}}~\citep{tan2025simvpv2}; Earth-system architectures \textit{\textbf{GraphCast}}~\citep{lam2023learning} and \textit{\textbf{CirT}}~\citep{liu2025cirt}; and ocean-specific models \textit{\textbf{WenHai}}~\citep{cui2025forecasting}, \textit{\textbf{NeuralOM}}~\citep{gao2026neuralom}, and \textit{\textbf{DSLCast}}~\citep{wu2026pnpcorrector}. We retrain SimVP, GraphCast, and CirT under the same protocol and use official pretrained weights for WenHai, NeuralOM, and DSLCast.

\subsection{Probabilistic Ocean Simulation}
\begin{table*}[h!]
    \centering
    \caption{Quantitative ocean simulation results over 93 variables. Lower RMSE/CRPS, higher ACC, and SSR closer to 1 are better. Values are mean{\tiny$\pm$std}; red/orange/blue mark the best/second-/third-best results. $*$ denotes Perlin-noise ensemble modeling for deterministic baselines.}
    \vspace{-5pt}
    \small
    \label{tab:mainres}
    \setlength{\tabcolsep}{1.7pt}
    \begin{tabular}{l|cccc|cccc}
        \toprule
        \multirow{2}{*}{Method}
        & \multicolumn{4}{c|}{90-Day}
        & \multicolumn{4}{c}{120-Day} \\
        \cmidrule(lr){2-5} \cmidrule(lr){6-9}
        & RMSE $\downarrow$ & ACC $\uparrow$ & CRPS $\downarrow$ & SSR $\rightarrow 1$
        & RMSE $\downarrow$ & ACC $\uparrow$ & CRPS $\downarrow$ & SSR $\rightarrow 1$ \\
        \midrule
        SimVP      & 0.932{\tiny$\pm$0.005} & 0.088{\tiny$\pm$0.006} & \textemdash & \textemdash
                  & 0.932{\tiny$\pm$0.012} & 0.058{\tiny$\pm$0.005} & \textemdash & \textemdash \\
        SimVP*     & 0.811{\tiny$\pm$0.006} & 0.104{\tiny$\pm$0.005} & \thirdcell{0.485{\tiny$\pm$0.004}} & 0.582{\tiny$\pm$0.012}
                  & 0.807{\tiny$\pm$0.010} & 0.073{\tiny$\pm$0.006} & \thirdcell{0.489{\tiny$\pm$0.007}} & 0.602{\tiny$\pm$0.015} \\
        GraphCast  & 0.921{\tiny$\pm$0.005} & 0.198{\tiny$\pm$0.003} & \textemdash & \textemdash
                  & 1.031{\tiny$\pm$0.004} & 0.139{\tiny$\pm$0.003} & \textemdash & \textemdash \\
        GraphCast* & 0.854{\tiny$\pm$0.003} & 0.218{\tiny$\pm$0.002} & 0.525{\tiny$\pm$0.002} & 0.452{\tiny$\pm$0.005}
                  & 0.900{\tiny$\pm$0.006} & 0.165{\tiny$\pm$0.003} & 0.547{\tiny$\pm$0.003} & 0.596{\tiny$\pm$0.007} \\
        CirT       & 2.823{\tiny$\pm$0.022} & 0.016{\tiny$\pm$0.003} & \textemdash & \textemdash
                  & 2.902{\tiny$\pm$0.024} & 0.009{\tiny$\pm$0.004} & \textemdash & \textemdash \\
        CirT*      & 2.415{\tiny$\pm$0.036} & 0.019{\tiny$\pm$0.002} & 1.087{\tiny$\pm$0.019} & \thirdcell{0.624{\tiny$\pm$0.026}}
                  & 2.455{\tiny$\pm$0.036} & 0.010{\tiny$\pm$0.005} & 1.127{\tiny$\pm$0.018} & \thirdcell{0.648{\tiny$\pm$0.029}} \\
        WenHai     & 0.875{\tiny$\pm$0.006} & 0.291{\tiny$\pm$0.006} & \textemdash & \textemdash
                  & 0.939{\tiny$\pm$0.009} & 0.231{\tiny$\pm$0.011} & \textemdash & \textemdash \\
        NeuralOM   & \secondcell{0.682{\tiny$\pm$0.009}} & \secondcell{0.393{\tiny$\pm$0.009}} & \textemdash & \textemdash
                  & \secondcell{0.710{\tiny$\pm$0.005}} & \thirdcell{0.319{\tiny$\pm$0.007}} & \textemdash & \textemdash \\
        DSLCast    & 0.880{\tiny$\pm$0.007} & 0.333{\tiny$\pm$0.009} & \textemdash & \textemdash
                  & 0.915{\tiny$\pm$0.009} & 0.294{\tiny$\pm$0.013} & \textemdash & \textemdash \\
        DSLCast*   & \thirdcell{0.749{\tiny$\pm$0.005}} & \thirdcell{0.385{\tiny$\pm$0.007}} & \secondcell{0.424{\tiny$\pm$0.003}} & \secondcell{0.659{\tiny$\pm$0.005}}
                  & \thirdcell{0.737{\tiny$\pm$0.006}} & \secondcell{0.362{\tiny$\pm$0.011}} & \secondcell{0.410{\tiny$\pm$0.003}} & \secondcell{0.779{\tiny$\pm$0.007}} \\
        \midrule
        \model{}   & \bestcell{0.681{\tiny$\pm$0.005}} & \bestcell{0.421{\tiny$\pm$0.005}} & \bestcell{0.370{\tiny$\pm$0.002}} & \bestcell{0.707{\tiny$\pm$0.004}}
                  & \bestcell{0.670{\tiny$\pm$0.003}} & \bestcell{0.400{\tiny$\pm$0.007}} & \bestcell{0.368{\tiny$\pm$0.002}} & \bestcell{0.784{\tiny$\pm$0.005}} \\
        \bottomrule
    \end{tabular}
    \vspace{-10pt}
\end{table*}

\begin{figure}[h!]
    \centering
    \includegraphics[width=\linewidth]{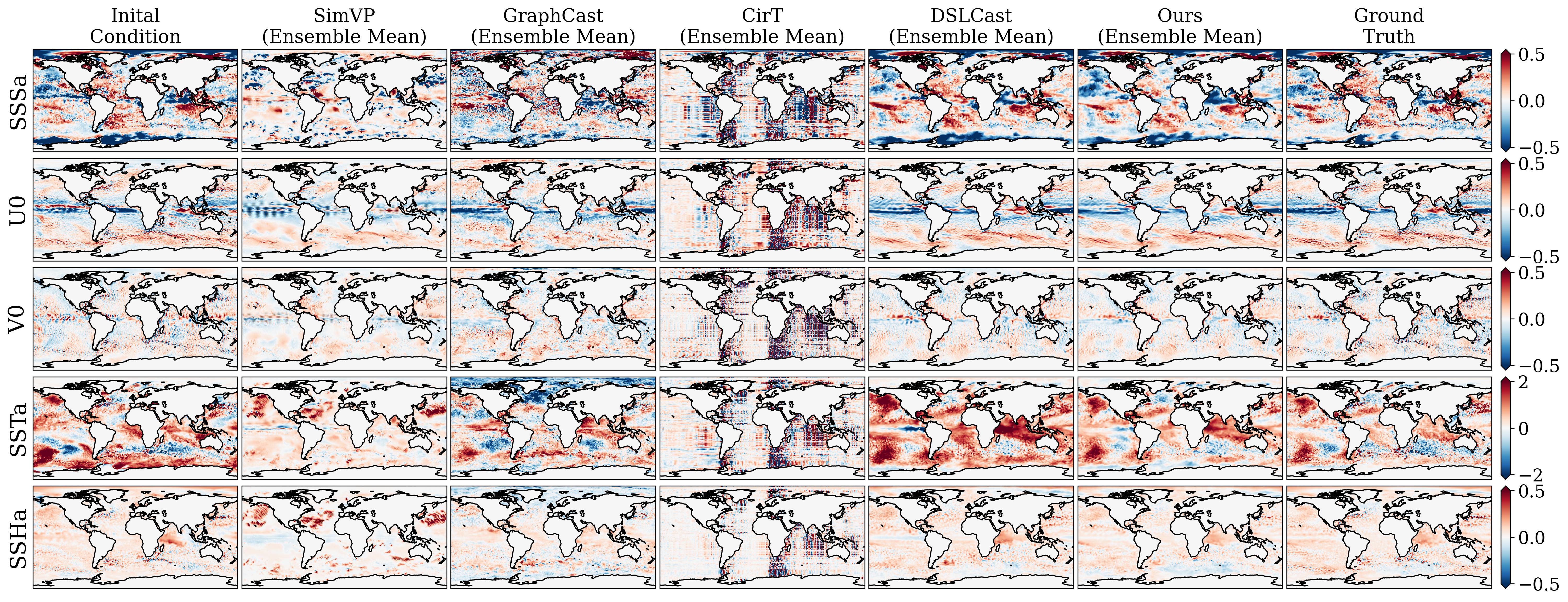}
    \vspace{-20pt}
    \caption{120-day simulation results for several important surface variables. Our method achieves better physical consistency and yields results that are closest to the ground truth.}
    \vspace{-10pt}
    \label{fig:visual_120d}
\end{figure}

\begin{figure*}[t!]
    \centering
    \includegraphics[width=\textwidth]{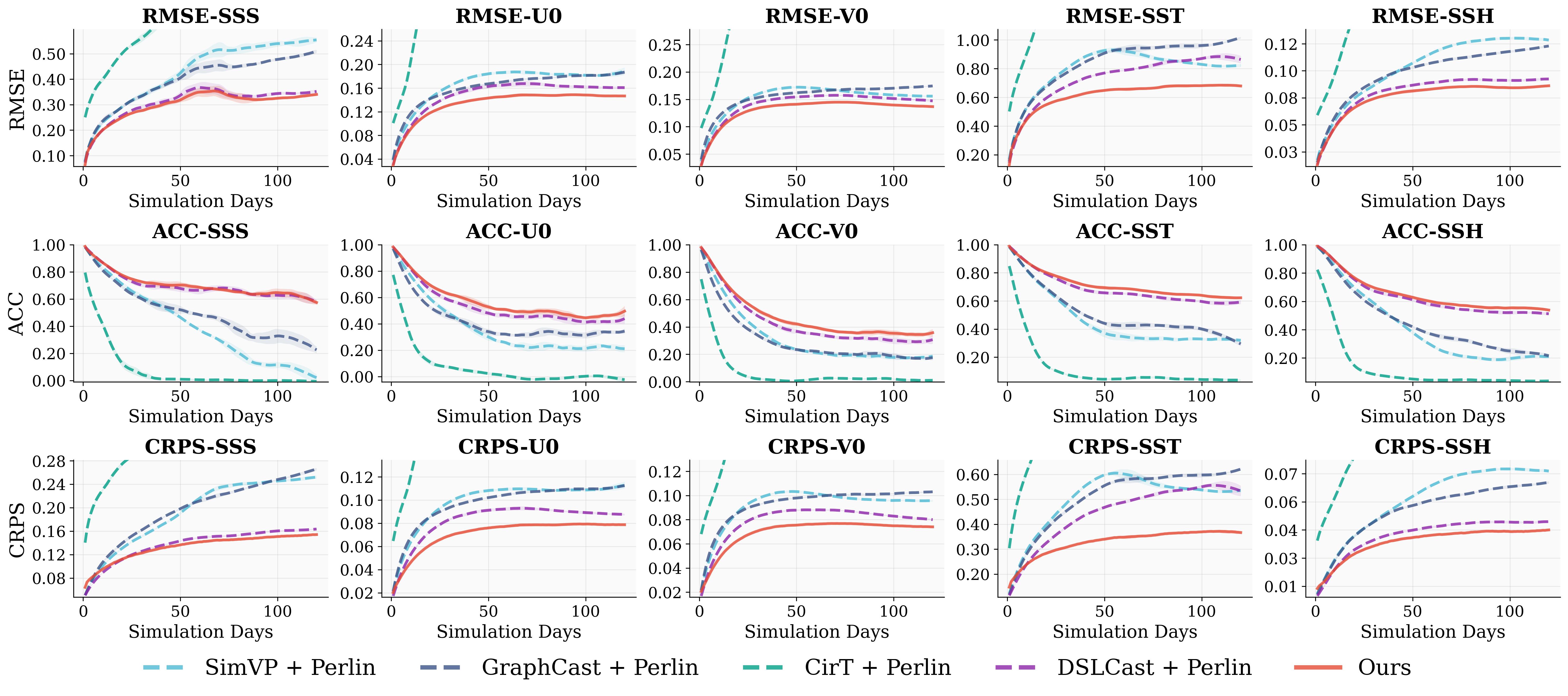}
    \vspace{-15pt}
    \caption{The RMSE, ACC and CRPS results of several important ocean surface variables.}
    \label{fig:single_vars}
    \vspace{-20pt}
\end{figure*}

\begin{figure*}[ht!]
    \centering
    \includegraphics[width=\textwidth]{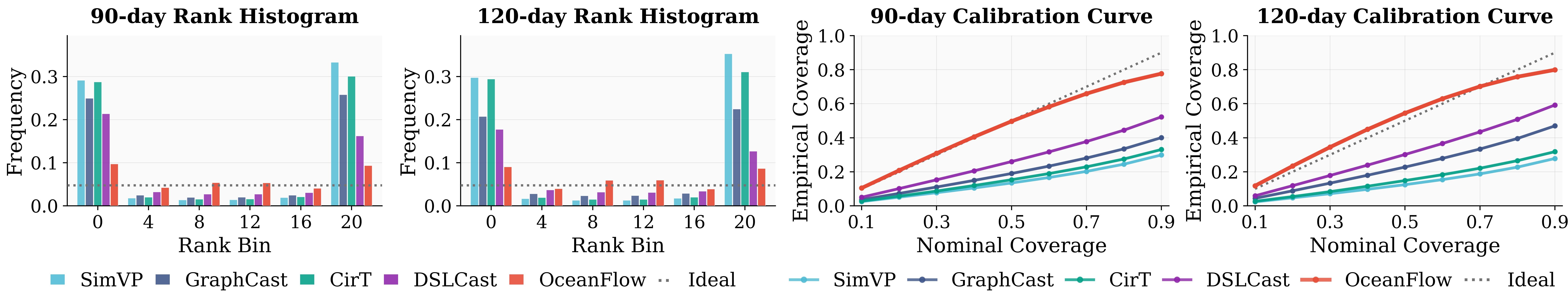}
    \vspace{-15pt}
    \caption{Rank histograms and calibration curves of ocean simulations.}
\label{fig:rank_histograms_calibration_simulation}
    \vspace{-10pt}
\end{figure*}
Table~\ref{tab:mainres} reports ocean simulation results over 20 initial conditions (ICs) under prescribed atmospheric forcing, separating ocean-state uncertainty from atmospheric forecasting error. \model{} achieves the lowest ensemble-mean RMSE and highest ACC, together with the lowest CRPS, across all lead times. Although Perlin perturbations improve ensemble-mean RMSE over the corresponding deterministic trajectories, substantial gaps in ACC and CRPS remain relative to \model{}, particularly at longer lead times. This suggests that perturbation-induced spread alone may not adequately represent the state-dependent uncertainty of ocean dynamics. These results support the central premise of coherent flow matching: jointly refining the conditional residual distribution improves probabilistic quality while preserving an accurate ensemble center. Figure~\ref{fig:single_vars} shows that the gains extend across important sea-surface variables throughout the simulation range. The 120-day fields in Figure~\ref{fig:visual_120d} provide a complementary spatial comparison: the ensemble mean of \model{} remains sharp and spatially organized, while several baselines lose fine-scale structure or exhibit pronounced spatial distortions. Figure~\ref{fig:rank_histograms_calibration_simulation} and Table~\ref{tab:reliability_diag} further assess ensemble reliability. \model{} yields rank distributions closer to uniform and empirical coverage closer to nominal levels, with the lowest RankDev and CalibErr at all simulation horizons.

\subsection{Probabilistic Ocean Forecasting}

\begin{table*}[t!]
    \centering
    \caption{Quantitative reliability diagnostics for ocean simulation and forecasting. RankDev is computed as the L2 deviation using all 21 rank bins. Lower values are better.}
    \vspace{-10pt}
    \small
    \label{tab:reliability_diag}
    \setlength{\tabcolsep}{1.9pt}
    \begin{tabular}{l|cc|cc||cc|cc}
        \toprule
        \multirow{3}{*}{Method}
        & \multicolumn{4}{c||}{Simulation}
        & \multicolumn{4}{c}{Forecasting} \\
        \cmidrule(lr){2-5} \cmidrule(lr){6-9}
        & \multicolumn{2}{c|}{90-Day}
        & \multicolumn{2}{c||}{120-Day}
        & \multicolumn{2}{c|}{20-Day}
        & \multicolumn{2}{c}{30-Day} \\
        \cmidrule(lr){2-3} \cmidrule(lr){4-5} \cmidrule(lr){6-7} \cmidrule(lr){8-9}
        & RankDev$\downarrow$ & CalibErr$\downarrow$
        & RankDev$\downarrow$ & CalibErr$\downarrow$
        & RankDev$\downarrow$ & CalibErr$\downarrow$
        & RankDev$\downarrow$ & CalibErr$\downarrow$ \\
        \midrule
        SimVP*     
        & 0.396{\tiny$\pm$0.005} & 0.355{\tiny$\pm$0.003}
        & 0.416{\tiny$\pm$0.006} & 0.366{\tiny$\pm$0.004}
        & 0.328{\tiny$\pm$0.002} & \thirdcell{0.292{\tiny$\pm$0.002}}
        & 0.338{\tiny$\pm$0.002} & \thirdcell{0.307{\tiny$\pm$0.002}} \\
        GraphCast* 
        & \thirdcell{0.309{\tiny$\pm$0.002}} & \thirdcell{0.299{\tiny$\pm$0.001}}
        & \thirdcell{0.254{\tiny$\pm$0.002}} & \thirdcell{0.261{\tiny$\pm$0.002}}
        & \thirdcell{0.326{\tiny$\pm$0.002}} & 0.294{\tiny$\pm$0.001}
        & 0.352{\tiny$\pm$0.002} & 0.317{\tiny$\pm$0.001} \\
        CirT*      
        & 0.368{\tiny$\pm$0.010} & 0.337{\tiny$\pm$0.008}
        & 0.381{\tiny$\pm$0.009} & 0.343{\tiny$\pm$0.009}
        & 0.426{\tiny$\pm$0.004} & 0.372{\tiny$\pm$0.002}
        & 0.390{\tiny$\pm$0.003} & 0.351{\tiny$\pm$0.002} \\
        DSLCast*   
        & \secondcell{0.215{\tiny$\pm$0.002}} & \secondcell{0.230{\tiny$\pm$0.002}}
        & \secondcell{0.163{\tiny$\pm$0.003}} & \secondcell{0.189{\tiny$\pm$0.002}}
        & \secondcell{0.288{\tiny$\pm$0.003}} & \secondcell{0.261{\tiny$\pm$0.003}}
        & \secondcell{0.306{\tiny$\pm$0.002}} & \secondcell{0.281{\tiny$\pm$0.002}} \\
        \midrule
        \model{}   
        & \bestcell{0.088{\tiny$\pm$0.001}} & \bestcell{0.032{\tiny$\pm$0.001}}
        & \bestcell{0.093{\tiny$\pm$0.001}} & \bestcell{0.041{\tiny$\pm$0.001}}
        & \bestcell{0.079{\tiny$\pm$0.002}} & \bestcell{0.104{\tiny$\pm$0.003}}
        & \bestcell{0.106{\tiny$\pm$0.002}} & \bestcell{0.124{\tiny$\pm$0.003}} \\
        \bottomrule
    \end{tabular}
    \vspace{-10pt}
\end{table*}

\begin{table*}[t!]
    \centering
    \caption{Probabilistic ocean forecasting results for 93 ocean variables.}
    \vspace{-10pt}
    \small
    \label{tab:forecasting}
    \setlength{\tabcolsep}{2pt}
    \begin{tabular}{l|cccc|cccc}
        \toprule
        \multirow{2}{*}{Method}
        & \multicolumn{4}{c|}{20-Day}
        & \multicolumn{4}{c}{30-Day} \\
        \cmidrule(lr){2-5} \cmidrule(lr){6-9}
        & RMSE $\downarrow$ & ACC $\uparrow$ & CRPS $\downarrow$ & SSR $\rightarrow 1$
        & RMSE $\downarrow$ & ACC $\uparrow$ & CRPS $\downarrow$ & SSR $\rightarrow 1$ \\
        \midrule
        SimVP*     
        & \thirdcell{0.667{\tiny$\pm$0.006}} & \thirdcell{0.602{\tiny$\pm$0.004}} & \thirdcell{0.408{\tiny$\pm$0.003}} & 0.221{\tiny$\pm$0.002}
        & 0.774{\tiny$\pm$0.008} & \thirdcell{0.465{\tiny$\pm$0.007}} & 0.475{\tiny$\pm$0.004} & 0.271{\tiny$\pm$0.005} \\
        GraphCast* 
        & 0.684{\tiny$\pm$0.005} & 0.523{\tiny$\pm$0.005} & 0.414{\tiny$\pm$0.003} & 0.258{\tiny$\pm$0.002}
        & \thirdcell{0.749{\tiny$\pm$0.004}} & 0.395{\tiny$\pm$0.005} & \thirdcell{0.461{\tiny$\pm$0.002}} & 0.269{\tiny$\pm$0.004} \\
        CirT*      
        & 1.640{\tiny$\pm$0.020} & 0.096{\tiny$\pm$0.003} & 0.787{\tiny$\pm$0.012} & \secondcell{0.354{\tiny$\pm$0.008}}
        & 1.984{\tiny$\pm$0.037} & 0.044{\tiny$\pm$0.003} & 0.903{\tiny$\pm$0.016} & \secondcell{0.434{\tiny$\pm$0.012}} \\
        DSLCast*   
        & \secondcell{0.624{\tiny$\pm$0.008}} & \secondcell{0.641{\tiny$\pm$0.008}} & \secondcell{0.360{\tiny$\pm$0.004}} & \thirdcell{0.324{\tiny$\pm$0.005}}
        & \secondcell{0.712{\tiny$\pm$0.008}} & \secondcell{0.521{\tiny$\pm$0.007}} & \secondcell{0.416{\tiny$\pm$0.004}} & \thirdcell{0.357{\tiny$\pm$0.006}} \\
        \midrule
        \model{}   
        & \bestcell{0.589{\tiny$\pm$0.007}} & \bestcell{0.660{\tiny$\pm$0.006}} & \bestcell{0.311{\tiny$\pm$0.004}} & \bestcell{0.514{\tiny$\pm$0.006}}
        & \bestcell{0.661{\tiny$\pm$0.008}} & \bestcell{0.550{\tiny$\pm$0.007}} & \bestcell{0.357{\tiny$\pm$0.003}} & \bestcell{0.483{\tiny$\pm$0.007}} \\
        \bottomrule
    \end{tabular}
    \vspace{-10pt}
\end{table*}

\begin{figure*}[t!]
    \centering
    \includegraphics[width=\textwidth]{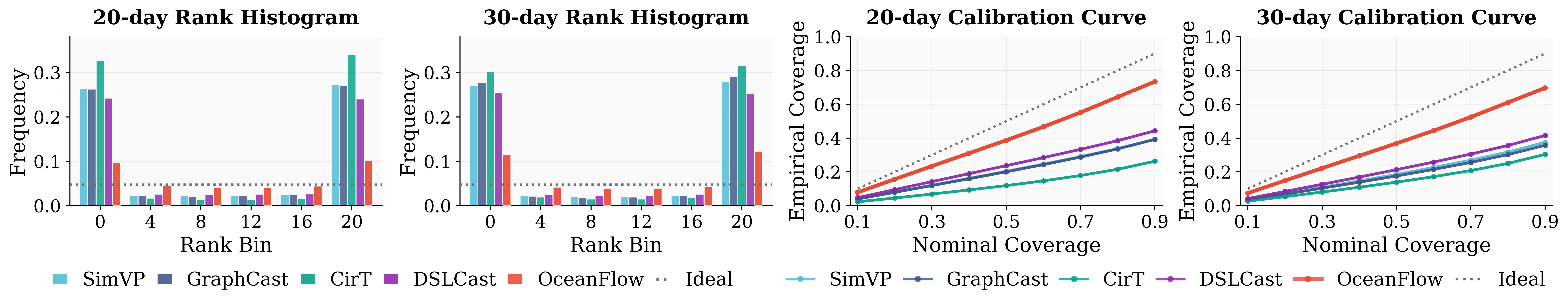}
    \vspace{-20pt}
    \caption{Rank histograms and calibration curves of probabilistic ocean forecasting.}
\label{fig:forecast_rank_histograms_calibration_selected_bins_bar}
    \vspace{-20pt}
\end{figure*}

To assess \model{}'s real-world application ability, we drive the ocean model with atmospheric forecasts produced by OneForecast~\citep{gao2025oneforecast}. OneForecast is initialized from the atmospheric state paired with each ocean IC and rolled out at 6-hour resolution to provide forcings. The results in Table~\ref{tab:forecasting}, averaged over 20 ICs, show that \model{} improves both ensemble-mean skill and CRPS at 20-day and 30-day lead times. The reliability gains observed in simulation persist under predicted forcing: Figure~\ref{fig:forecast_rank_histograms_calibration_selected_bins_bar} and Table~\ref{tab:reliability_diag} (right) show more uniform rank histograms and better-calibrated uncertainty estimates than the perturbation-based baselines.

\subsection{Generalization and ENSO Diagnosis}

We evaluate out-of-sample generalization in 2024 using 10 ICs. Global metrics are averaged over 93 variables, while ENSO metrics are computed from the Niño-3.4 SST anomaly index over the same ICs. We compare against DSLCast+Perlin, the strongest perturbation-based ensemble baseline in the main simulation comparison. Table~\ref{tab:generalization_enso_2024} shows improvements across global simulation metrics and Niño-3.4 RMSE and CRPS at both 90-day and 120-day horizons. These gains extend beyond the 2020 test year, supporting generalization in both global ocean states and large-scale tropical Pacific variability. As a preliminary extension toward coupled forecasting, we drive \model{} with atmospheric states generated by an independently running two-way PnP-Corrector system~\citep{wu2026pnpcorrector}. At the 90-day horizon, this configuration improves all reported global and Niño-3.4 metrics relative to the one-way forcing configuration. These findings suggest that end-to-end probabilistic two-way coupling could further improve ocean forecasting.

\begin{table*}[t!]
    \centering
    \caption{Out-of-sample 2024 simulation and ENSO diagnosis. The best results are shown in \textbf{bold}.}
    \vspace{-10pt}
    \label{tab:generalization_enso_2024}
    \setlength{\tabcolsep}{1.6pt}
    \small
    \begin{tabular}{llcccccc|cc}
        \toprule
        \multirow{2}{*}{Lead} & \multirow{2}{*}{Model}
        & \multicolumn{6}{c|}{Global ocean simulation}
        & \multicolumn{2}{c}{Niño-3.4 diagnosis} \\
        \cmidrule(lr){3-8} \cmidrule(lr){9-10}
        & & RMSE$\downarrow$ & ACC$\uparrow$ & CRPS$\downarrow$ & RankDev$\downarrow$ & SSR$\rightarrow 1$ & CalibErr$\downarrow$
        & RMSE$\downarrow$ & CRPS$\downarrow$ \\
        \midrule
        90d & DSLCast* & 0.784{\tiny$\pm$0.003} & 0.399{\tiny$\pm$0.004} & 0.449{\tiny$\pm$0.001} & 0.216{\tiny$\pm$0.001} & 0.665{\tiny$\pm$0.004} & 0.232{\tiny$\pm$0.001} & 0.6181 & 0.5007 \\
        90d & \model{} & \textbf{0.698{\tiny$\pm$0.003}} & \textbf{0.430{\tiny$\pm$0.003}} & \textbf{0.383{\tiny$\pm$0.001}} & \textbf{0.095{\tiny$\pm$0.001}} & \textbf{0.744{\tiny$\pm$0.001}} & \textbf{0.034{\tiny$\pm$0.000}} & \textbf{0.1451} & \textbf{0.0901} \\
        120d & DSLCast* & 0.779{\tiny$\pm$0.003} & 0.377{\tiny$\pm$0.003} & 0.435{\tiny$\pm$0.001} & 0.159{\tiny$\pm$0.001} & 0.790{\tiny$\pm$0.005} & 0.188{\tiny$\pm$0.001} & 0.7977 & 0.6879 \\
        120d & \model{} & \textbf{0.686{\tiny$\pm$0.003}} & \textbf{0.416{\tiny$\pm$0.003}} & \textbf{0.382{\tiny$\pm$0.001}} & \textbf{0.102{\tiny$\pm$0.001}} & \textbf{0.836{\tiny$\pm$0.005}} & \textbf{0.054{\tiny$\pm$0.001}} & \textbf{0.2304} & \textbf{0.1496} \\
        \bottomrule
    \end{tabular}
    \vspace{-10pt}
\end{table*}

\begin{table*}[t!]
    \centering
    \caption{Out-of-sample 2024 probabilistic ocean forecasting and ENSO prediction results.}
    \vspace{-10pt}
    \label{tab:forecasting_enso_2024}
    \setlength{\tabcolsep}{0.9pt}
    \small
    \begin{tabular}{lcccccc|cc}
        \toprule
        \multirow{2}{*}{Model}
        & \multicolumn{6}{c|}{Global ocean forecasting}
        & \multicolumn{2}{c}{Niño-3.4 diagnosis} \\
        \cmidrule(lr){2-7} \cmidrule(lr){8-9}
        & RMSE$\downarrow$ & ACC$\uparrow$ & CRPS$\downarrow$ & RankDev$\downarrow$ & SSR$\rightarrow 1$ & CalibErr$\downarrow$
        & RMSE$\downarrow$ & CRPS$\downarrow$ \\
        \midrule
        DSLCast* one-way
        & 0.867{\tiny$\pm$0.027} & 0.258{\tiny$\pm$0.028} & 0.495{\tiny$\pm$0.013} & 0.245{\tiny$\pm$0.009} & 0.605{\tiny$\pm$0.020} & 0.253{\tiny$\pm$0.006}
        & 0.9736 & 0.7933 \\
        \model{} one-way
        & 0.796{\tiny$\pm$0.027} & 0.257{\tiny$\pm$0.033} & 0.436{\tiny$\pm$0.012} & 0.134{\tiny$\pm$0.005} & 0.678{\tiny$\pm$0.019} & 0.057{\tiny$\pm$0.005}
        & 0.9118 & 0.7173 \\
        \model{} two-way
        & \textbf{0.762{\tiny$\pm$0.014}} & \textbf{0.327{\tiny$\pm$0.013}} & \textbf{0.419{\tiny$\pm$0.006}} & \textbf{0.116{\tiny$\pm$0.002}} & \textbf{0.699{\tiny$\pm$0.009}} & \textbf{0.045{\tiny$\pm$0.001}}
        & \textbf{0.6898} & \textbf{0.4643} \\
        \bottomrule
    \end{tabular}
    \vspace{-10pt}
\end{table*}

\begin{figure*}[t!]
    \centering
    \includegraphics[width=\textwidth]{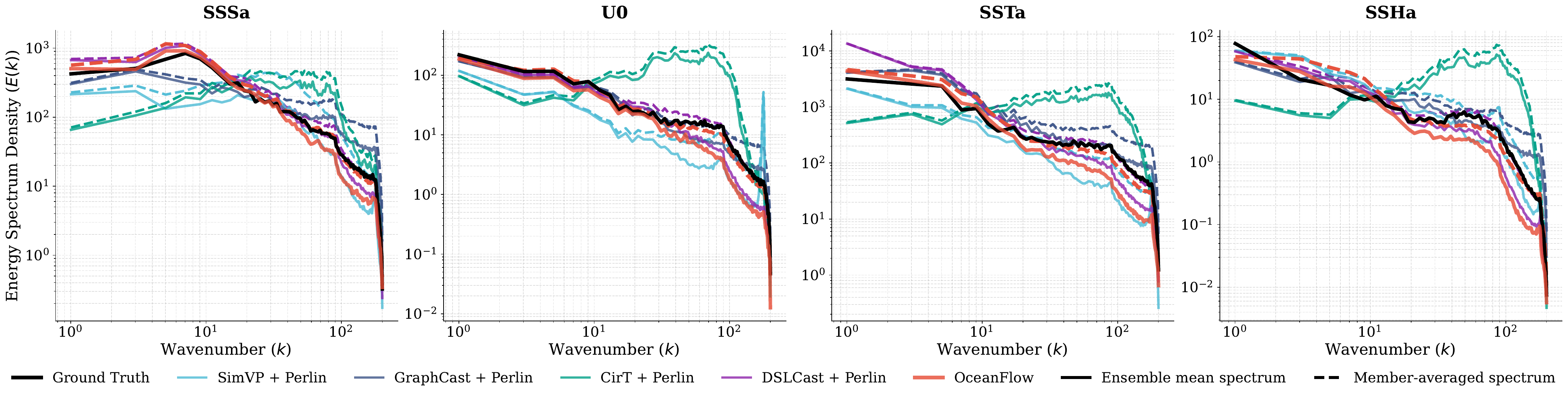}
    \vspace{-20pt}
    \caption{Spectral fidelity of ensemble mean and member-averaged 120-day ocean simulations}
    \vspace{-20pt}
    \label{fig:spectrum_120d}
\end{figure*}
\subsection{Analysis of Physical Consistency}

\begingroup
\setlength{\intextsep}{5pt}
\setlength{\columnsep}{10pt}
\begin{wrapfigure}{r}{0.55\textwidth}
    \vspace{-5pt}
    \centering
    \includegraphics[width=\linewidth]{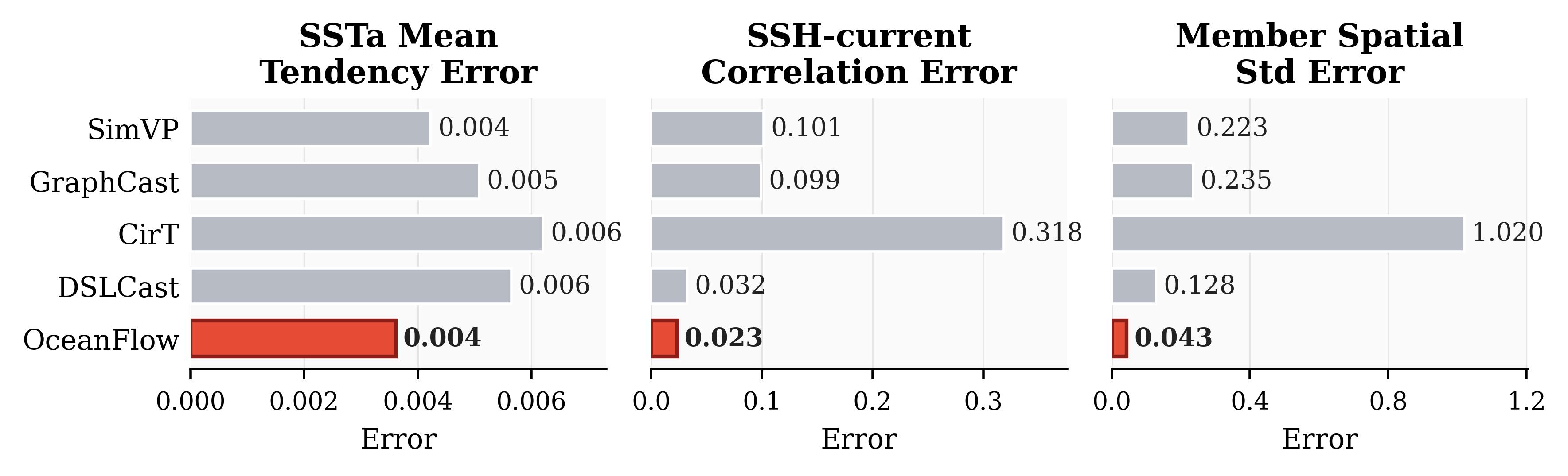}
    \setlength{\abovecaptionskip}{4pt}
    \setlength{\belowcaptionskip}{0pt}
    \vspace{-10pt}
    \caption{Physical diagnostics of the generated ensemble members and ensemble mean.}
    \vspace{-5pt}
    \label{fig:physical_diagnostics_compact}
\end{wrapfigure}
Aggregate scores alone do not fully reveal whether the generated ensemble members are physically plausible. Moreover, in a well-formed ensemble system, the ensemble mean should not be interpreted as a single realized ocean state: as lead time increases, averaging over plausible but phase-shifted members can naturally smooth unpredictable high-wavenumber components\citep{leith1974theoretical}. We therefore separate ensemble-mean structure from member-level realism in Figure~\ref{fig:spectrum_120d}.  For readability, we visualize representative surface variables covering typical near-surface dynamical and thermodynamic ocean processes. Solid lines show spectra of the ensemble mean, while dashed lines show member-averaged spectra computed by first evaluating each of the 20 ensemble members separately and then averaging their spectra over members and 20 ICs. The member-averaged spectra of \model{} are overall closer to the ground truth across these variables, indicating that individual ensemble members preserve realistic multiscale ocean variability. Meanwhile, the ensemble-mean spectra show larger deviations from the ground truth at high wavenumbers. This behavior is consistent with a well-calibrated ensemble, whose mean averages out unpredictable spectral components rather than reproducing the full spectrum of a single realization. In addition to the spectral comparison, Figure~\ref{fig:physical_diagnostics_compact} reports three targeted physical diagnostics at the 120-day simulation horizon, covering SSTa mean tendency consistency, SSH-current multivariate consistency, and member-level spatial variability. \model{} achieves the lowest error on all three diagnostics, further indicating that it better preserves large-scale thermal evolution, dynamical SSH-current relationships, and realistic spatial variability of individual ensemble members.
\par
\endgroup

\subsection{Ablation Study}

To validate the architectural design of \model{}, we conduct ablation studies through 90-day simulations over 10 ICs, as reported in Table~\ref{tab:oceanflow_ablation}. All variants are trained same as the main method. The deterministic variants are compared against FDO: (1) \textit{w/o Factorized Dynamics} replaces the factorized dynamics block with a plain convolutional block, substantially degrading simulation skill. (2) \textit{w/o State Response} removes the state-conditioned response, leading to poorer accuracy and highlighting its contribution to the factorized transition. (3) \textit{w/o Adaptive State Response} substitutes a fixed persistence addition for the learned response, partially recovering performance but remaining inferior to FDO. This supports adapting the state response rather than prescribing a fixed additive connection. The probabilistic comparisons further isolate the role of coherent residual transport: (4) \textit{FDO+Perlin} uses the same deterministic backbone with prescribed perturbations but yields poorer probabilistic scores than \model{}, showing that the gains cannot be attributed to the backbone alone. (5) \textit{w/o mean correction/re-centering} degrades both ensemble-mean accuracy and reliability, supporting explicit control of the predictive center. (6) \textit{w/o learned pathwise flow} also reduces probabilistic performance, demonstrating the benefit of learned residual transport beyond the perturbation reference. Together, these results support the complementary roles of FDO, mean anchoring, and coherent residual transport in producing accurate and reliable ocean ensembles.

\subsection{Training Protocol and Efficiency Analysis}

\begingroup
\setlength{\intextsep}{5pt}
\setlength{\columnsep}{6pt}
\begin{wraptable}{r}{0.52\textwidth}
\vspace{-5pt}
    \centering
    \caption{Parameter comparison for retrained models.}
    \vspace{-10pt}
    \label{tab:params}
    \setlength{\tabcolsep}{2.2pt}
    \small
    \begin{tabular}{l|cccc}
        \toprule
        Models & SimVP & GraphCast & CirT & OceanFlow \\
        \midrule
        Params (M) & 78.54 & 40.82 & 61.56 & 37.45 \\
        \bottomrule
    \end{tabular}
    \vspace{-2pt}
\end{wraptable}

We train SimVP, GraphCast, CirT, and the deterministic FDO backbone for 50 epochs on 8 NVIDIA A100 GPUs under the same framework. The learning rate starts at $10^{-3}$ and decays to zero with a cosine schedule. For \model{}, residual transport is trained for 3 additional epochs using $M=10$ residual particles per sample and $M_b=20$ reference members for mean anchoring. At inference, both the reference and transported ensembles contain 20 members ($M_b=M=20$). Table~\ref{tab:params} restricts the parameter comparison to models retrained under the common protocol; official pretrained baselines use different training configurations, \model{} has the fewest trainable parameters.
\par
\endgroup

\begin{table*}[t]
    \centering
    \caption{Ablation results of OceanFlow. The best results are shown in \textbf{bold}.}
    \vspace{-5pt}
    \small
    \label{tab:oceanflow_ablation}
    \setlength{\tabcolsep}{1.8pt}
    \begin{tabular}{l|cccccc}
        \toprule
        Variants & RMSE$\downarrow$ & ACC$\uparrow$ & CRPS$\downarrow$ & RankDev$\downarrow$ & SSR$\rightarrow 1$ & CalibErr$\downarrow$ \\
        \midrule
        w/o Factorized Dynamics
        & 1.054{\tiny$\pm$0.008}
        & 0.188{\tiny$\pm$0.009}
        & \textemdash
        & \textemdash
        & \textemdash
        & \textemdash \\
        w/o State Response
        & 0.934{\tiny$\pm$0.062}
        & 0.072{\tiny$\pm$0.090}
        & \textemdash
        & \textemdash
        & \textemdash
        & \textemdash \\
        w/o Adaptive State Response
        & 0.844{\tiny$\pm$0.007}
        & 0.322{\tiny$\pm$0.010}
        & \textemdash
        & \textemdash
        & \textemdash
        & \textemdash \\
        FDO
        & 0.825{\tiny$\pm$0.009}
        & 0.352{\tiny$\pm$0.007}
        & \textemdash
        & \textemdash
        & \textemdash
        & \textemdash \\
        FDO+Perlin
        & 0.715{\tiny$\pm$0.005}
        & 0.406{\tiny$\pm$0.003}
        & 0.403{\tiny$\pm$0.002}
        & 0.204{\tiny$\pm$0.001}
        & 0.625{\tiny$\pm$0.004}
        & 0.225{\tiny$\pm$0.001} \\
        w/o mean correction/re-centering
        & 0.692{\tiny$\pm$0.005}
        & 0.420{\tiny$\pm$0.003}
        & 0.376{\tiny$\pm$0.002}
        & 0.101{\tiny$\pm$0.001}
        & 0.690{\tiny$\pm$0.003}
        & 0.059{\tiny$\pm$0.001} \\
        w/o learned pathwise flow
        & 0.697{\tiny$\pm$0.004}
        & 0.417{\tiny$\pm$0.002}
        & 0.379{\tiny$\pm$0.002}
        & 0.109{\tiny$\pm$0.001}
        & 0.665{\tiny$\pm$0.004}
        & 0.081{\tiny$\pm$0.002} \\
        \model{}
        & \textbf{0.684{\tiny$\pm$0.005}}
        & \textbf{0.425{\tiny$\pm$0.002}}
        & \textbf{0.371{\tiny$\pm$0.002}}
        & \textbf{0.088{\tiny$\pm$0.000}}
        & \textbf{0.708{\tiny$\pm$0.003}}
        & \textbf{0.032{\tiny$\pm$0.000}} \\
        \bottomrule
    \end{tabular}
    \vspace{-15pt}
\end{table*}

%% file: sec/5_conclusions.tex
\section{Conclusion and Limitations}
We present \model{}, a coherent flow matching framework for probabilistic ocean modeling. By coupling member residuals through the evolving ensemble context, \model{} learns a coherent state-conditioned pathwise law that refines ensemble diversity around a controlled mean anchor. FDO parameterizes both the mean anchor and residual-flow velocity through a unified composition of factorized dynamics and learned state-conditioned response. Experiments demonstrate improved ensemble-mean accuracy, probabilistic reliability, and physical-consistency diagnostics, supporting coherent flow matching as an effective approach to ocean uncertainty modeling. Current ocean forecasting relies on the forcings produced by external model, extending the framework to end-to-end probabilistic ocean-atmosphere coupling may further improve ocean forecasting accuracy.

%% file: sec/X_suppl.tex
\makeatletter
\newcommand{\appendixlocaltoc}{%
    \begingroup
    \hypersetup{linkcolor=black}%
    \setcounter{tocdepth}{2}%
    \section*{Appendix Contents}%
    \@starttoc{atoc}%
    \endgroup
}
\let\appendix@original@addcontentsline\addcontentsline
\renewcommand{\addcontentsline}[3]{%
    \appendix@original@addcontentsline{#1}{#2}{#3}%
    \def\appendix@target{#1}%
    \def\appendix@toc{toc}%
    \ifx\appendix@target\appendix@toc
        \appendix@original@addcontentsline{atoc}{#2}{#3}%
    \fi
}
\makeatother

\appendixlocaltoc
\clearpage

\section{Proof of the Coherent Conditional Law}
\label{app:coherence_proof}

\begin{proof}[Proof of Lemma~\ref{lem:coherent_transport}]
\itshape
Fix trained parameters and ensemble size $M$, with $\mathcal X=\mathbb R^d$ denoting valid ocean-grid coordinates. State and context spaces carry their Borel sigma-algebras, and $\rho_0^M(\cdot\mid c)$ is a probability kernel. Assume joint measurability of the velocity, anchor, reference composition, calibration, and context encoding. For continuous transport, assume continuity in $\tau$ and global Lipschitz continuity in $Z$, uniformly on $[0,1]$ with constant $L<\infty$ for fixed $(c,C)$, including the shared context. With $\|Z\|_M^2=M^{-1}\sum_m\|z^{(m)}\|_2^2$, the flow satisfies
\begin{equation}
    \begin{aligned}
    \left\|V_\theta^M(Z,C,\tau)-V_\theta^M(Z',C,\tau)\right\|_M
    &\le L\left\|Z-Z'\right\|_M,\\
    \left\|\Phi_{\theta,C,\tau}^{M}(Z_0)-
    \Phi_{\theta,C,\tau}^{M}(Z'_0)\right\|_M
    &\le e^{L\tau}\left\|Z_0-Z'_0\right\|_M.
    \end{aligned}
    \label{eq:proof_flow_stability}
\end{equation}
Picard--Lindel\"of gives existence and uniqueness; Gr\"onwall bounds sensitivity in flow time $\tau$, not autoregressive lead time. For initial law $\nu_0^M(\cdot\mid c,C)$, the joint residual law obeys the weak continuity equation,
\begin{equation}
    \rho_{\theta,\tau}^{M}(\cdot\mid c,C)
    =
    \left(\Phi_{\theta,C,\tau}^{M}\right)_{\#}
    \nu_0^{M}(\cdot\mid c,C),
    \qquad
    \partial_\tau\rho_{\theta,\tau}^{M}
    +
    \nabla_Z\!\cdot\!\left(
        \rho_{\theta,\tau}^{M}V_{\theta}^{M}
    \right)
    =0.
    \label{eq:proof_joint_continuity}
\end{equation}
All member paths are components of this joint flow. For $m\neq j$, the chain rule at differentiable points and fixed $(c,C)$ gives
\begin{equation}
    D_{z^{(j)}}
    \left[V_{\theta}^{M}(Z,C,\tau)\right]^{(m)}
    =
    D_{S}v_{\theta}\,D_{z^{(j)}}S(Z).
    \label{eq:proof_cross_member_dependence}
\end{equation}
A nonzero cross derivative makes the velocity locally nonseparable: $S(Z)$ enables cross-member interaction beyond parameter sharing, although its learned strength need not be nonzero. The mean, standard deviation, and mean absolute deviation remain Lipschitz even at nondifferentiable points.

Shared parameters and symmetric summaries give $V_\theta^M(\Pi Z,\Pi C,\tau)=\Pi V_\theta^M(Z,C,\tau)$; uniqueness transfers this joint-relabeling symmetry to the flow.

Center control follows from the projection itself. For $U=(u^{(1)},\ldots,u^{(M)})$ and $\bar u=M^{-1}\sum_m u^{(m)}$,
\begin{equation}
    \begin{aligned}
    \mathcal{R}^{M}(U)^{(m)}&=u^{(m)}-\bar u,
    &\frac{1}{M}\sum_{m=1}^{M}
    \left[\mathcal{R}^{M}(U)\right]^{(m)}&=0,\\
    \left[\mathcal{R}^{M}(U)\right]^{(m)}-
    \left[\mathcal{R}^{M}(U)\right]^{(j)}
    &=u^{(m)}-u^{(j)}.
    \end{aligned}
    \label{eq:proof_projection_properties}
\end{equation}
Centering therefore preserves the pairwise contrasts of its input. Taking $U=\mathcal{K}_{\theta,c}^{M}(\mathcal{T}_{\theta,C}^{M}(Z_0),\Xi)$ after reference composition, calibration, and stabilization yields
\begin{equation}
    \frac{1}{M}\sum_{m=1}^{M}
    \left[G_{\theta,c}^{M}(\Xi)\right]^{(m)}
    =\mu_\theta(c,\Xi)
    +\frac{1}{M}\sum_{m=1}^{M}
    \left[\mathcal{R}^{M}(U)\right]^{(m)}
    =\mu_\theta(c,\Xi).
    \label{eq:proof_pathwise_mean_anchor}
\end{equation}
The identity holds for every $\Xi$, even when the anchor varies with the stochastic reference. Joint measurability makes $G_{\theta,c}^{M}$ measurable in $(c,\Xi)$, so its pushforward of $\rho_0^M(\cdot\mid c)$ is a probability kernel.

For continuous transport, $\mathcal T_{\theta,C}^{M}=\Phi_{\theta,C,1}^{M}$. Each implemented update in Eq.~(\ref{eq:euler_transport}), including masking and clipping, is measurable and permutation equivariant. Finite composition therefore induces its own conditional law with the same mean identity, without ordinary differential equation regularity or equality to the continuous law.

Finally, condition on a reference rollout $r$ and its forcing, with remaining stochastic inputs $\Omega$. Composing the transition-wise generators and context encodings defines $\mathcal G_{\theta,r}^{M}$ with
\begin{equation}
    Q_{\theta,0:T}^{M}(\cdot\mid r)
    =
    \left(\mathcal{G}_{\theta,r}^{M}\right)_{\#}
    \rho_{\Omega}^{M}(\cdot\mid r),
    \qquad
    \frac{1}{M}\sum_{m=1}^{M}
    \widehat{\mathbf{O}}_{t}^{(m)}
    =\mu_{\theta,t},
    \quad t=1,\ldots,T.
    \label{eq:proof_pathspace_law}
\end{equation}
This law describes reference-guided ensembles, not flow-driven autoregressive particles. Coherence concerns jointly conditioned residual paths and exact center control at each transition; the result neither requires persistent memory nor establishes calibration or long-horizon skill.
\end{proof}

\clearpage
\input{sec/2_Related_work}

\section{Evaluation Metrics}
\label{app:metrics}

Evaluation distinguishes the accuracy of the ensemble center from the quality of the distribution around it. Root mean square error (RMSE) and anomaly correlation coefficient (ACC) assess the ensemble mean; the continuous ranked probability score (CRPS), rank deviation (RankDev), spread--skill ratio (SSR), and calibration error (CalibErr) assess complementary aspects of the ensemble distribution. Lower RMSE, CRPS, RankDev, CalibErr, and physical-diagnostic errors are better; higher ACC and Ni\~no-3.4 correlation are better; and SSR is optimal at one. Ocean-specific diagnostics further examine regional variability, multivariate consistency, and member-level structure. Let $x$ and $y$ denote a predicted and target ocean field, let $c$ index the evaluated ocean channels, and let $(i,j)$ index latitude and longitude. The latitude weight is
\begin{equation}
    w_i
    =
    \frac{H\cos(\phi_i)}
    {\sum_{r=1}^{H}\cos(\phi_r)},
    \label{eq:appendix_lat_weight}
\end{equation}
where $H$ is the number of evaluated latitudes and $\phi_i$ is the latitude at row $i$. Spatial statistics are computed over valid ocean points within the relevant evaluation region. Where used, $\epsilon>0$ is a small constant for numerical stability.

\paragraph{Root mean square error (RMSE).}
RMSE is computed on normalized ocean variables and then averaged over channels:
\begin{equation}
    \mathrm{RMSE}
    =
    \frac{1}{C}
    \sum_{c=1}^{C}
    \sqrt{
    \operatorname{nanmean}_{i,j}
    \left[
    w_i\left(\hat{x}_{c,i,j}-y_{c,i,j}\right)^2
    \right]
    } .
    \label{eq:appendix_rmse}
\end{equation}

\paragraph{Anomaly correlation coefficient (ACC).}
ACC is computed on anomalies relative to the climatological mean. Let $\hat{x}'=\hat{x}-x^{\mathrm{clim}}$ and $y'=y-y^{\mathrm{clim}}$. The channel-wise anomaly correlation is
\begin{equation}
    \mathrm{ACC}_c
    =
    \frac{
    \sum_{i,j} w_i \hat{x}'_{c,i,j} y'_{c,i,j}
    }
    {
    \sqrt{
    \left(\sum_{i,j} w_i(\hat{x}'_{c,i,j})^2\right)
    \left(\sum_{i,j} w_i(y'_{c,i,j})^2\right)}
    },
    \qquad
    \mathrm{ACC}=\frac{1}{C}\sum_{c=1}^{C}\mathrm{ACC}_c .
    \label{eq:appendix_acc}
\end{equation}

\paragraph{Continuous ranked probability score (CRPS).}
For an ensemble $\{\hat{x}^{(m)}\}_{m=1}^{M}$, we use the empirical ensemble form~\citep{matheson1976scoring,hersbach2000decomposition,gneiting2007strictly},
\begin{equation}
    \begin{aligned}
    \mathrm{CRPS}
    &=
    \frac{1}{C}\sum_{c=1}^{C}
    \operatorname{nanmean}_{i,j}
    \left[
    w_i
    \left(
    \frac{1}{M}\sum_{m=1}^{M}
    \left|\hat{x}_{c,i,j}^{(m)}-y_{c,i,j}\right|
    -
    \frac{1}{2M^2}\sum_{m,n=1}^{M}
    \left|\hat{x}_{c,i,j}^{(m)}-\hat{x}_{c,i,j}^{(n)}\right|
    \right)
    \right].
    \end{aligned}
    \label{eq:appendix_crps_metric}
\end{equation}

\paragraph{Rank deviation (RankDev).}
At every channel--grid point where the target and all $M$ members are valid, the target rank is
\begin{equation}
    r_{c,i,j}
    =
    \sum_{m=1}^{M}
    \mathbf{1}\!\left\{
    \hat{x}_{c,i,j}^{(m)}<y_{c,i,j}
    \right\}
    \in\{0,\ldots,M\}.
    \label{eq:appendix_rank}
\end{equation}
Let $\mathcal{V}$ denote the set of valid channel--grid points. The latitude-weighted rank histogram and its deviation from the uniform histogram are
\begin{equation}
    h_b
    =
    \frac{
    \sum_{(c,i,j)\in\mathcal{V}}
    w_i\mathbf{1}\{r_{c,i,j}=b\}
    }{
    \sum_{(c,i,j)\in\mathcal{V}}w_i+\epsilon
    },
    \qquad
    \mathrm{RankDev}
    =
    \left[
    \sum_{b=0}^{M}
    \left(h_b-\frac{1}{M+1}\right)^2
    \right]^{1/2}.
    \label{eq:appendix_rankdev}
\end{equation}
We use $M=20$, giving 21 rank bins. Lower RankDev indicates a rank distribution closer to uniform.

\paragraph{Spread--skill ratio (SSR).}
Let $\bar{x}_{c,i,j}=M^{-1}\sum_m\hat{x}_{c,i,j}^{(m)}$ and let
\begin{equation}
    s_{c,i,j}
    =
    \left[
    \frac{1}{M-1}
    \sum_{m=1}^{M}
    \left(
    \hat{x}_{c,i,j}^{(m)}-\bar{x}_{c,i,j}
    \right)^2
    \right]^{1/2}
    \label{eq:appendix_ensemble_spread}
\end{equation}
be the sample standard deviation across members. Denoting the latitude-weighted, channel-averaged RMSE in Eq.~\eqref{eq:appendix_rmse} by $\mathcal{E}_w(\cdot,\cdot)$, we compute
\begin{equation}
    \mathrm{SSR}
    =
    \frac{
    \mathcal{E}_w(s,0)
    }{
    \mathcal{E}_w(\bar{x},y)+\epsilon
    }.
    \label{eq:appendix_ssr}
\end{equation}
The numerator measures ensemble spread and the denominator measures ensemble-mean skill; values closer to one are better.

\paragraph{Calibration error (CalibErr).}
For nominal central coverage $p$, let $q_\alpha(c,i,j)$ be the empirical $\alpha$-quantile of the $M$ members at $(c,i,j)$. Its latitude-weighted empirical coverage is
\begin{equation}
    \widehat{c}(p)
    =
    \frac{
    \sum_{(c,i,j)\in\mathcal{V}}
    w_i\,
    \mathbf{1}\!\left\{
    q_{(1-p)/2}(c,i,j)
    \le y_{c,i,j}\le
    q_{(1+p)/2}(c,i,j)
    \right\}
    }{
    \sum_{(c,i,j)\in\mathcal{V}}w_i+\epsilon
    }.
    \label{eq:appendix_empirical_coverage}
\end{equation}
Using $\mathcal{P}=\{0.1,0.2,\ldots,0.9\}$, the scalar calibration error is
\begin{equation}
    \mathrm{CalibErr}
    =
    \frac{1}{|\mathcal{P}|}
    \sum_{p\in\mathcal{P}}
    \left|\widehat{c}(p)-p\right|.
    \label{eq:appendix_caliberr}
\end{equation}
Lower CalibErr indicates closer agreement between nominal and empirical coverage. Unless stated otherwise, all six global metrics are first computed separately for each initial condition and then summarized across initial conditions.

\paragraph{Ni\~no-3.4 diagnosis.}
We assess El Ni\~no--Southern Oscillation (ENSO) variability using the Ni\~no-3.4 sea surface temperature (SST) anomaly index. Let $\mathcal{R}_{3.4}$ denote the region $5^\circ\mathrm{S}$--$5^\circ\mathrm{N}$, $170^\circ\mathrm{W}$--$120^\circ\mathrm{W}$. For initial condition $q$, the member and verifying indices are computed directly from the raw SST anomaly fields as
\begin{equation}
    n_q^{(m)}
    =
    \frac{1}{|\mathcal{R}_{3.4}|}
    \sum_{(i,j)\in\mathcal{R}_{3.4}}
    \hat{s}_{q,i,j}^{(m)},
    \qquad
    n_q^\star
    =
    \frac{1}{|\mathcal{R}_{3.4}|}
    \sum_{(i,j)\in\mathcal{R}_{3.4}}
    s_{q,i,j}^\star,
    \label{eq:appendix_nino_index}
\end{equation}
without adding a climate mean, normalizing variables, or applying latitude weights. Let $\bar n_q=M^{-1}\sum_m n_q^{(m)}$, $\mu_{\bar n}=I^{-1}\sum_q\bar n_q$, and $\mu_{n^\star}=I^{-1}\sum_q n_q^\star$. At each reported lead time, across $I=10$ initial conditions, we compute
\begin{equation}
    \begin{aligned}
    \mathrm{RMSE}_{3.4}
    &=
    \left[
    \frac{1}{I}\sum_{q=1}^{I}
    \left(\bar n_q-n_q^\star\right)^2
    \right]^{1/2},\\
    \mathrm{Corr}_{3.4}
    &=
    \frac{
    \sum_q(\bar n_q-\mu_{\bar n})(n_q^\star-\mu_{n^\star})
    }{
    \sqrt{
    \sum_q(\bar n_q-\mu_{\bar n})^2
    \sum_q(n_q^\star-\mu_{n^\star})^2
    }},\\
    \mathrm{CRPS}_{3.4}
    &=
    \frac{1}{I}\sum_{q=1}^{I}
    \left[
    \frac{1}{M}\sum_{m=1}^{M}|n_q^{(m)}-n_q^\star|
    -
    \frac{1}{2M^2}\sum_{m,n=1}^{M}|n_q^{(m)}-n_q^{(n)}|
    \right].
    \end{aligned}
    \label{eq:appendix_nino_metrics}
\end{equation}

The three physical-consistency diagnostics below distinguish the evolution of the ensemble center from the spatial variability of individual members. The first two assess central-state consistency, whereas the third computes spatial variability before averaging across members. Here SSS and SSH denote sea surface salinity and height, while U0 and V0 denote surface zonal and meridional velocities. The suffix \textit{a} indicates anomaly fields: SSSa, SSTa, SSHa, U0a, and V0a. We use the latitude-weighted mean and standard deviation over valid ocean points $\mathcal{O}$,
\begin{equation}
    \langle a\rangle_w
    =
    \frac{\sum_{(i,j)\in\mathcal{O}}w_i a_{i,j}}
    {\sum_{(i,j)\in\mathcal{O}}w_i},
    \qquad
    \sigma_w(a)
    =
    \left[
    \frac{\sum_{(i,j)\in\mathcal{O}}
    w_i(a_{i,j}-\langle a\rangle_w)^2}
    {\sum_{(i,j)\in\mathcal{O}}w_i}
    \right]^{1/2}.
    \label{eq:appendix_weighted_ocean_moments}
\end{equation}

\paragraph{SSTa mean tendency error.}
Let $\bar{s}_t$ be the ensemble-mean SST anomaly field and $y_t^s$ its verifying field at lead day $t$. Over the $T=120$ day trajectory, we compute
\begin{equation}
    E_{\mathrm{SSTa}}
    =
    \frac{1}{T-1}
    \sum_{t=2}^{T}
    \frac{
    \left|
    (\langle\bar{s}_t\rangle_w-\langle\bar{s}_{t-1}\rangle_w)
    -(\langle y_t^s\rangle_w-\langle y_{t-1}^s\rangle_w)
    \right|
    }{
    \sigma_w(y_t^s)+\epsilon
    }.
    \label{eq:appendix_ssta_tendency}
\end{equation}
The latitude-weighted global mean summarizes the evolution of surface thermal anomalies. Comparing its daily increments tests whether the ensemble center follows the observed aggregate tendency throughout the rollout, rather than merely matching the final field. This provides a conservation-related consistency check on the mean trajectory, not a test of exact energy conservation: SST anomalies alone do not determine the full ocean heat budget.

\paragraph{SSH--current correlation error.}
For a valid grid set $\mathcal{B}$, let $\widetilde a_{i,j}=a_{i,j}-\langle a\rangle_{w,\mathcal B}$ and define $\widetilde b$ analogously. The latitude-weighted correlation is
\begin{equation}
    \operatorname{corr}_w(a,b;\mathcal{B})
    =
    \frac{
    \sum_{(i,j)\in\mathcal{B}}
    w_i\widetilde a_{i,j}\widetilde b_{i,j}
    }{
    \left[
    \sum_{(i,j)\in\mathcal{B}}w_i\widetilde a_{i,j}^{2}
    \sum_{(i,j)\in\mathcal{B}}w_i\widetilde b_{i,j}^{2}
    \right]^{1/2}}.
    \label{eq:appendix_weighted_corr}
\end{equation}
Here $\mathcal{B}$ contains ocean points satisfying $10^\circ\le|\phi_i|\le80^\circ$ whose immediate grid neighbors are also valid. The superscripts $\mathrm{pred}$ and $\mathrm{gt}$ denote the ensemble-mean and verifying fields, respectively. For $q\in\{\mathrm{pred},\mathrm{gt}\}$, using discrete grid gradients of SSH anomaly,
\begin{equation}
    r_u^q
    =
    \operatorname{corr}_w
    \left(\mathrm{U0a}^{q},-\partial_y\mathrm{SSHa}^{q};\mathcal{B}\right),
    \qquad
    r_v^q
    =
    \operatorname{corr}_w
    \left(\mathrm{V0a}^{q},\partial_x\mathrm{SSHa}^{q};\mathcal{B}\right),
    \label{eq:appendix_ssh_current_corr}
\end{equation}
and the ensemble-mean consistency error is
\begin{equation}
    E_{\mathrm{SSH\text{-}current}}
    =
    \frac{1}{2}
    \left(
    |r_u^{\mathrm{pred}}-r_u^{\mathrm{gt}}|
    +
    |r_v^{\mathrm{pred}}-r_v^{\mathrm{gt}}|
    \right).
    \label{eq:appendix_ssh_current_error}
\end{equation}
SSH gradients and near-surface currents are dynamically related away from the equator. Comparing their correlations with those of the verifying state tests whether the ensemble center retains this cross-variable relationship, which separate field-wise errors cannot establish. The diagnostic measures agreement in aggregate multivariate structure; it does not impose a geostrophic constraint or measure its pointwise residual.

\paragraph{Member spatial standard-deviation error.}
For the $K=5$ surface anomaly variables $\{\mathrm{SSSa},\mathrm{U0a},\mathrm{V0a},\mathrm{SSTa},\mathrm{SSHa}\}$, we compute
\begin{equation}
    E_{\mathrm{std}}
    =
    \frac{1}{K}
    \sum_{k=1}^{K}
    \left|
    \log
    \frac{
    \frac{1}{M}\sum_{m=1}^{M}
    \sigma_w(x_k^{(m)})+\epsilon
    }{
    \sigma_w(y_k)+\epsilon
    }
    \right|.
    \label{eq:appendix_member_spatial_std}
\end{equation}
Unlike the first two diagnostics, this metric computes spatial variability within each member before averaging across members. It therefore complements center-based evaluation by detecting a deficit or excess in the average member variability across thermodynamic and dynamical surface fields. The absolute log-ratio penalizes deficient and excessive variability symmetrically. This amplitude-based diagnostic complements spectral analysis without requiring each member to reproduce the verifying field's spatial phase.

\section{Data Details}
\label{appendix:data}

\subsection{Dataset}

We conduct the experiments on GLORYS12 reanalysis data, which provides daily mean ocean variables covering latitudes between $-80^\circ$ and $90^\circ$ from 1993 to the present. The data can be downloaded from \url{https://data.marine.copernicus.eu}. We use 1993--2017 for training, 2018--2019 for validation, and 2020 for the main test set. The additional 2024 evaluation in Table~\ref{tab:generalization_enso_2024} assesses generalization beyond this test year without changing the training split. The ocean variables used in our experiments include four depth-dependent variables and one surface variable. The four depth-dependent variables are Sea salinity (S), Sea stream zonal velocity (U), Sea stream meridional velocity (V), and Sea temperature (T), each of which contains 23 depth levels. These depth levels correspond to 0.49 m, 2.65 m, 5.08 m, 7.93 m, 11.41 m, 15.81 m, 21.60 m, 29.44 m, 40.34 m, 55.76 m, 77.85 m, 92.32 m, 109.73 m, 130.67 m, 155.85 m, 186.13 m, 222.48 m, 266.04 m, 318.13 m, 380.21 m, 453.94 m, 541.09 m, and 643.57 m. The surface variable used in this work is Sea surface height (SSH). To improve computational efficiency, we downsample all variables to a spatial resolution of $1^\circ$ using bilinear interpolation, resulting in an original grid size of $H=181$ and $W=360$. To better adapt the data to the input requirements of different model architectures, we retain the first 180 latitude grids and use data with a spatial size of $180 \times 360$ in the experiments. Four ERA5 variables provide atmospheric forcing: 10 metre u wind component (U10M), 10 metre v wind component (V10M), 2 metre temperature (T2M), and mean sea level pressure (MSLP). Each ocean transition uses these four fields at the current and next day, yielding eight atmospheric input channels. Ocean states are daily means, while forcing fields are sampled at 12:00 UTC. ERA5 data are available from \url{https://cds.climate.copernicus.eu}. Table~\ref{tab:appendix_data} summarizes the variables used in the main experiments.

\begin{table*}[ht]
\centering
\caption{The data details used in this work.}
\label{tab:appendix_data}
\setlength{\tabcolsep}{2pt}
\begin{tabular}{ccccccc}
\toprule
\textbf{Type}        & \textbf{Full name}                 & \textbf{Abbreviation} & \textbf{Layers} & \textbf{Time}      & \textbf{Dt} & \textbf{Spatial Res.}\\
\midrule
Atmosphere & 10 metre u wind component     & U10M         & 1      & 1993--2020 & 24h             & $1^\circ$               \\
Atmosphere & 10 metre v wind component     & V10M         & 1      & 1993--2020 & 24h             & $1^\circ$               \\
Atmosphere & 2 metre temperature           & T2M          & 1      & 1993--2020 & 24h             & $1^\circ$               \\
Atmosphere & Mean sea level pressure       & MSLP         & 1      & 1993--2020 & 24h             & $1^\circ$               \\
\midrule
Ocean     & Sea salinity                   & S            & 23     & 1993--2020 & 24h             & $1^\circ$               \\
Ocean     & Sea stream zonal velocity      & U            & 23     & 1993--2020 & 24h             & $1^\circ$               \\
Ocean     & Sea stream meridional velocity & V            & 23     & 1993--2020 & 24h             & $1^\circ$               \\
Ocean     & Sea temperature                & T            & 23     & 1993--2020 & 24h             & $1^\circ$               \\
Ocean     & Sea surface height             & SSH          & 1      & 1993--2020 & 24h             & $1^\circ$               \\
\midrule
Static    & Land-sea mask                  & LSM          & ---    & ---        & ---             & $1^\circ$               \\
\bottomrule
\end{tabular}
\vspace{-9pt}
\end{table*}

\subsection{Data Preprocessing}

Different ocean variables and forcing variables have different physical meanings, units, and numerical ranges. To reduce the influence of these scale differences, we normalize each variable before feeding it into the model. This preprocessing step enables the model to focus more on learning the spatiotemporal evolution of ocean states rather than adapting to the magnitude discrepancies among variables. We first compute daily climatological means for ocean salinity, temperature, and sea-surface height using the first 365 days of each year in the 1993--2017 training period. The resulting climatology has dimensions $365 \times 47 \times 181 \times 360$, corresponding to days, variables, latitude, and longitude, respectively. The 47 variables comprise salinity and temperature at 23 depth levels each, together with sea-surface height. These variables are converted into anomalies by subtracting their corresponding daily climatological means. We then compute variable-wise means and standard deviations over the same training period, using anomaly fields for these 47 variables and the original fields for the remaining variables, and apply these statistics for normalization. For forcing variables, means and standard deviations are computed from ERA5 data over 1959--2017. Each variable is normalized independently using its own mean and standard deviation. Specifically, for each variable, we first subtract its corresponding mean and then divide the result by its corresponding standard deviation. For land regions, the original ocean variables contain `nan' values. Before inputting the data into the model, these missing values are filled with zeros after normalization. The Land-sea mask is used as a static variable to indicate the distribution of land and ocean regions.

\section{Detailed Method}
\label{app:detailed_method}

Mean anchoring and ensemble-conditioned residual transport jointly define the coherent conditional law at each ocean transition. The deterministic reference establishes the basis for the anchor, while the residual flow learns the distribution around it. Separate FDO instances parameterize the reference transition and residual velocity.

\subsection{Mean Anchoring}

The deterministic reference and residual flow use the same normalized ocean representation. Let $q_t$ denote a raw ocean channel and $q^{\mathrm{clim}}_{d(t)}$ its daily climatology for salinity, temperature, and sea-surface height; for velocity channels, set $q^{\mathrm{clim}}_{d(t)}=0$. With channel-wise training statistics defining $\mathrm{Norm}(\cdot)$ and ocean mask $M_o$, the normalized state is
\begin{equation}
    \tilde{q}_t
    =
    M_o \odot
    \mathrm{Norm}\!\left(q_t - q^{\mathrm{clim}}_{d(t)}\right).
    \label{eq:appendix_norm}
\end{equation}
Normalization is evaluated on valid ocean points, with land values subsequently set to zero. The masked expressions here and below use this zero-fill convention. Member conditions introduce structured perturbations only in the normalized ocean-state channels,
\begin{equation}
    \tilde{q}_{t}^{(m)}
    =
    M_o \odot
    \left(\tilde{q}_t + \sigma_{\eta} d_q\,\epsilon_q^{(m)}\right),
    \label{eq:appendix_member_condition}
\end{equation}
where $d_q$ sets the variable- and depth-dependent scale. Concatenating the perturbed state with the shared forcing gives $c_t^{(m)}$ in Eq.~\eqref{eq:base_anchor}. The perturbation reference supplies initial ensemble variability and a spread scale for subsequent distributional refinement. For smooth, spatially centered Perlin fields $\xi_q^{(m)}$, ocean-domain root mean square (RMS) normalization gives
\begin{equation}
    \epsilon_q^{(m)}
    =
    \frac{M_o \odot \xi_q^{(m)}}
    {
    \sqrt{
    \frac{\sum_{i,j} M_o(i,j)\left(\xi_q^{(m)}(i,j)\right)^2}
         {\sum_{i,j} M_o(i,j) + \varepsilon}
    }
    + \varepsilon
    } .
    \label{eq:appendix_perlin_norm}
\end{equation}
Salinity, temperature, velocity, and sea-surface height have separate base scales, with subsurface amplitudes decaying with depth. These relative scales are fixed, whereas the positive global amplitude $\sigma_\eta$ is learned jointly with the probabilistic model. Forcing channels and the land--sea mask remain unperturbed. The deterministic outputs define the reference center in Eq.~\eqref{eq:base_anchor}, which the mean branch adjusts through Eq.~\eqref{eq:mean_correction}. We set $\alpha_\mu=0.25$ in normalized variable space. The channel-wise ocean-domain projection is
\begin{equation}
    \mathcal P_o r
    =
    M_o \odot
    \left(
    r
    -
    \frac{\sum_{i,j} M_o(i,j) r(i,j)}
         {\sum_{i,j} M_o(i,j)}
    \right),
    \label{eq:appendix_zero_mean_correction}
\end{equation}
where each channel has at least one valid ocean cell. Applying $\mathcal P_o$ before and after the bounded nonlinearity preserves a zero spatial mean for the correction. This spatial constraint is distinct from ensemble recentering, which fixes the member average at $\mu_{t+1}$. The mean branch adapts the predictive center to the current state, while the final ensemble projection prevents residual transport from displacing it.

\subsection{Coherent State-Conditioned Pathwise Flow}
\label{app:flow_matching_details}
\label{app:set_statistics}

Coherent residual transport learns the joint distribution of member departures around the mean anchor. Suppressing the physical-time index within a transition, let $\bar{x}_{\mathrm{base}}=M^{-1}\sum_m x_{\mathrm{base}}^{(m)}$ denote the mean of the $M$ reference members paired with residual particles, and define
\begin{equation}
    b^{(m)}
    =
    x_{\mathrm{base}}^{(m)}
    -
    \bar{x}_{\mathrm{base}}.
    \label{eq:appendix_base_anomaly}
\end{equation}
The member-specific centers are $\mu+b^{(m)}$, and the reference anomalies also determine the spread reference. Training uses $M_b=20$ reference members to estimate the mean anchor and $M=10$ residual particles, each yielding one output member. The anomalies above are centered over the $M$ reference members paired with these particles. At evaluation, both the reference and output ensembles contain 20 members ($M_b=M=20$). Residual noise sampled independently of the reference anomalies initializes the transport,
\begin{equation}
    z_0^{(m)} = \sigma_0 \epsilon_0^{(m)} ,
    \label{eq:appendix_initial_residual}
\end{equation}
where $\epsilon_0^{(m)}$ is smooth, spatially centered noise with per-member RMS normalization followed by the prescribed variable- and depth-dependent scaling. Centering across members precedes normalization, so the initial particle set is not constrained to have an exactly zero empirical mean; Eq.~\eqref{eq:recenter} enforces this identity at the output. Let $e=y-\operatorname{sg}(\mu)$, with $\operatorname{sg}$ denoting stop-gradient. The spread target $s_{\mathrm{target}}$ is obtained by scaling and bounding the channel-wise ocean RMS of $e$ and broadcasting it over valid ocean points. With $r_{\mathrm{pat}}^{(m)}(e)$ denoting the scaled error pattern modulated by centered, RMS-normalized member coefficients, the transport target is
\begin{equation}
    z_1^{(m)}
    =
    (1-\gamma_p)s_{\mathrm{target}}\odot\epsilon_0^{(m)}
    +
    \gamma_p\,r_{\mathrm{pat}}^{(m)}(e),
    \qquad
    \frac{1}{M}\sum_m r_{\mathrm{pat}}^{(m)} = 0,
    \label{eq:appendix_target_residual}
\end{equation}
followed by ocean masking, with $\gamma_p=0.05$. The interpolation and velocity target in the main text use one shared $\tau$ for each training ensemble, so all member velocities are learned within the same joint interpolation state. Flow matching supervises this state-dependent transport, while the mean and probability objectives supervise the generated ocean ensemble.

The integrated paths used for distributional training and inference both use two updates with $\Delta\tau=1/2$. Each update recomputes the ensemble context, evaluates the shared velocity at $\tau_k=(k+1/2)\Delta\tau$, and applies the residual bound and mask through $\mathcal P_r$ in Eq.~\eqref{eq:euler_transport}. Pathwise thus refers to member residual evolution in flow time within each ocean transition. Let $\mathcal H_\theta$ denote residual centering, smoothing, and learned spread calibration, and let $\mathcal B_{\mu,B}$ apply member stabilization, recentering, and the reference-spread floor. The composition preceding the final projection is
\begin{equation}
    \widetilde Z_{N_\tau}
    =\mathcal K_{\theta,c}^{M}(Z_{N_\tau},\Xi)
    =\mathcal B_{\mu,B}\!\left(
    \mathbf 1_M\otimes\mu+B+\mathcal H_\theta(Z_{N_\tau};c)
    \right)-\mathbf 1_M\otimes\mu,
    \label{eq:appendix_residual_composition}
\end{equation}
where $B=(b^{(1)},\ldots,b^{(M)})$. Spread calibration blends the centered residual with a version normalized to the learned spread map. Stabilization bounds member departures, and the reference floor rescales channel-wise ensemble RMS when it falls below the prescribed fraction of reference spread. Final recentering in Eq.~\eqref{eq:recenter} preserves the anchor. Together, reference anomalies and learned residuals determine ensemble diversity around this controlled center.

The shared velocity couples these paths through the evolving residual ensemble. For $Z=\{z^{(m)}\}_{m=1}^{M}$, its symmetric context is
\begin{equation}
    S(Z)
    =
    \left[
    \frac{1}{M}\sum_m z^{(m)},\,
    \sqrt{\frac{1}{M}\sum_m (z^{(m)}-\bar{z})^2},\,
    \frac{1}{M}\sum_m |z^{(m)}-\bar{z}|
    \right].
    \label{eq:appendix_set_stats}
\end{equation}
The mean, standard deviation, and mean absolute deviation describe the current residual center and dispersion; they condition the velocity rather than impose fixed moment constraints. The mean and standard deviation enter each member's input directly, alongside a learned projection of all three summaries. Ensemble summaries evolve with the residual paths, whereas the ocean--forcing embedding is shared across members and recomputed at each ocean transition. Shared parameters and symmetric summaries yield
\begin{equation}
    \mathcal T_\theta^M(\Pi Z\mid\Pi C)
    =\Pi\mathcal T_\theta^M(Z\mid C),
    \label{eq:perm_equivariance}
\end{equation}
where particles and member conditions are permuted together. Masking, clipping, reference composition, and recentering preserve this symmetry. Coherence is expressed by the dependence of member velocities on the evolving ensemble context, as formalized in Eq.~\eqref{eq:proof_cross_member_dependence}; permutation equivariance ensures that this interaction does not depend on arbitrary member ordering.

\subsection{Factorized Dynamics Operator}
\label{app:velocity_operator}

FDO parameterizes conditional transitions through factorized dynamics and a learned state-conditioned response, following the operator composition in Eqs.~\eqref{eq:fdo_encoding}--\eqref{eq:fdo_closure}. Separate instances support the deterministic reference underlying the mean anchor and the residual-flow velocity. For the velocity instance, $\mathbf X^{(m)}=[z^{(m)},c_t^{(m)},\bar z,\sigma(Z),\mathcal J_\theta(S(Z))]$, where $\sigma(Z)$ is the pointwise member standard deviation and $\mathcal J_\theta$ projects the ensemble summaries. Its context $\mathbf C$ combines flow time with the current ocean--forcing embedding. The deterministic instance takes the ocean--forcing tensor and a learned static context embedding. Conditional patch encoding gives
\begin{equation}
    H_0
    =
    \mathcal E_\theta^{\mathrm{patch}}\!\left(
    \mathbf{X}^{(m)};\mathbf C
    \right)
    \in \mathbb{R}^{D \times H_p \times W_p}.
    \label{eq:appendix_patch_embed}
\end{equation}
Within each block, contextual modulation determines the representation on which factorized dynamics operate,
\begin{equation}
    U_\ell
    =
    a_\ell(\mathbf C)\odot\mathcal N_\ell(H_\ell)
    +
    b_\ell(\mathbf C),
    \label{eq:appendix_modulation}
\end{equation}
where $\mathcal N_\ell$ is normalization. A spatial transformation $\mathcal M_\ell$ and learned projection $\mathcal P_\ell$ induce $K$ factors of width $d=D/K$. Factor-specific evolution and cross-factor interaction then take the form
\begin{align}
    F_{\ell,k}
    &=
    [\mathcal P_\ell(\mathcal M_\ell(U_\ell))]_k,
    \qquad G_{\ell,k}=\mathcal T_{\ell,k}(F_{\ell,k}),
    \nonumber\\
    \widehat G_{\ell,k}
    &=
    \sum_{j=1}^{K}
    B_{\ell,kj} G_{\ell,j},
    \qquad
    B_\ell = I + \beta \tanh(A_\ell).
    \label{eq:appendix_factor_mix}
\end{align}
Factor-specific operators evolve distinct learned subspaces, and the near-identity matrix $B_\ell$ mediates their interactions. Its coefficients are learned but state independent; a state-dependent gate selects the contribution of each interacting factor at each location,
\begin{equation}
    \begin{aligned}
    \Gamma_\ell&=\operatorname{softmax}_{k}\!\left(\mathcal G_\ell(U_\ell)\right),\\
    H'_\ell&=H_\ell+\Psi_\ell\!\left(
    \operatorname{Concat}_{k}(\Gamma_{\ell,k}\odot\widehat G_{\ell,k})\right),\\
    H_{\ell+1}&=H'_\ell+\mathcal C_\ell(H'_\ell;\mathbf C).
    \end{aligned}
    \label{eq:appendix_spatial_gate}
\end{equation}
Here $\Psi_\ell$ integrates the weighted factors, and $\mathcal C_\ell$ supplies a normalized, context-modulated channel update; together they instantiate $\Phi_\ell$ in Eq.~\eqref{eq:fdo_factorized_transition}. Factor indices identify learned subspaces, without imposing a one-to-one correspondence with physical variables or ocean processes. These updates act in encoding and decoding streams connected by a learned width projection. Normalization and patch restoration through $\mathcal R_\theta$ yield the decoded field, whose state-conditioned response gives
\begin{equation}
    \widetilde{v}_\theta
    =
    \mathcal R_\theta(H_L),
    \qquad
    v_\theta
    =
    \left[
    \mathcal{I}
    +
    \gamma\mathcal{S}_{\theta}(\,\cdot\mid\mathbf{X}^{(m)})
    \right]
    \left(
    \widetilde{v}_\theta
    \right)
    =
    \widetilde{v}_\theta
    +
    \gamma\mathcal{S}_\theta
    \left(
    \widetilde{v}_\theta,\mathbf{X}^{(m)}
    \right).
    \label{eq:appendix_patch_restore}
\end{equation}
The learned response adapts the decoded factorized transition to its input condition: member residuals, ocean--forcing inputs, and ensemble context for the velocity instance, or the ocean--forcing state for the deterministic instance. Both instances follow the same operator composition but use separate parameters. The velocity output projection and response coefficient are initialized to zero, and the deterministic instance is pretrained before ensemble learning.

\subsection{Training Losses and Calibration}
\label{app:training_objectives}

Training first establishes the deterministic reference and then learns the mean-anchored residual distribution. The reference FDO models the normalized one-step transition
\begin{equation}
    \widehat{\mathbf O}_{t+1}
    =f_{\mathrm{det},\vartheta}(c_t)
    =\mathcal M_\vartheta(\mathbf O_t,\mathbf F_t,\mathbf F_{t+1}),
    \label{eq:deterministic_transition}
\end{equation}
using the relative $L_2$ risk in Eq.~\eqref{eq:deterministic_pretraining}. Let $\ell_{\mathrm{cw}}(a,y)$ denote its sample-level error, computed separately for each ocean channel and then averaged; the spatial norm is neither squared nor latitude weighted. The pretrained parameters $\vartheta$ remain frozen in the probabilistic stage, whose trainable parameters are denoted by $\theta$. Its flow matching and CRPS terms use the fixed-grid masked average
\begin{equation}
    \|u\|_{p,o}^{p}
    =
    \frac{1}{C_oHW}\sum_{q,i,j}|M_{o,q}(i,j)u_q(i,j)|^p .
    \label{eq:appendix_masked_norm}
\end{equation}
Land cells remain zero in this fixed-grid reduction. The expectation in Eq.~\eqref{eq:compact_training_objective} averages over training transitions, stochastic member inputs, and flow time. Its mean component combines central-state supervision with improvement over the frozen reference,
\begin{equation}
    \begin{aligned}
    \ell_{\mathrm{mean}}
    ={}&a_\mu\ell_{\mathrm{cw}}(\mu,y)
    +a_{\mathrm{ens}}\ell_{\mathrm{cw}}(\bar x,y)\\
    &+a_{\mathrm{imp}}\left[
    \|\bar x-y\|_{2,o}^{2}
    -\operatorname{sg}\!\left(\|\mu^{\mathrm{base}}-y\|_{2,o}^{2}\right)
    \right]_+,
    \end{aligned}
    \label{eq:appendix_mean_losses}
\end{equation}
where $\bar x=M^{-1}\sum_m\hat x^{(m)}$ and $[r]_+=\max(0,r)$. The coefficients absorb the mean-supervision weights, including the corrected-reference loss on $\mu^{\mathrm{base}}+(\bar x-\mu^{\mathrm{base}})$, which is algebraically identical to supervision of $\bar x$. Reference comparisons are fixed under differentiation. Distributional supervision combines empirical CRPS in Eq.~\eqref{eq:crps} with $\ell_{\mathrm{gain}}=[\mathsf S_M(\widehat X,y)-\operatorname{sg}(\mathsf S_M(X_{\mathrm{base}},y))]_+$. Gradients through the integrated residual paths connect the interpolation dynamics to terminal distributional quality. The $M^2$ denominator gives the empirical ensemble score, not its finite-member fair-score correction. Member-specific reference centers are detached in the probability and member losses.

The remaining terms supervise member error and the amplitude of the learned distribution. Let $\sigma_M(X)$ be the pointwise ensemble standard deviation with denominator $M$, $\mathcal H$ the spatial smoothing map, and $\hat s_\theta$ the learned spread field. For derived spread fields, write $\|u\|_{2,g}^2=(C_oHW)^{-1}\sum_{q,i,j}u_q(i,j)^2$ for the full-grid mean square; smoothing acts on masked fields without an additional mask before reduction. The spread objectives are
\begin{equation}
    \begin{aligned}
    \ell_{\mathrm{spread}}
    & =\|\mathcal H(\sigma_M(\widehat X))-s_{\mathrm{target}}\|_{2,g}^{2}
    +a_s\|\hat s_\theta-s_{\mathrm{target}}\|_{2,g}^{2},\\
    \ell_{\mathrm{anchor}}
    & =\|[\rho\,\sigma_M(X_{\mathrm{base}})-\sigma_M(\widehat X)]_+\|_{2,o}^{2}.
    \end{aligned}
    \label{eq:appendix_spread_losses}
\end{equation}
The regularizer $\ell_{\mathrm{reg}}$ combines these terms with the member error $M^{-1}\sum_m\ell_{\mathrm{cw}}(\hat x^{(m)},y)$. Spread targets use the detached center error; the training spread penalty acts pointwise, whereas the sampling floor in $\mathcal B_{\mu,B}$ acts on channel-wise ensemble RMS. These objectives jointly supervise the center, residual transport, and terminal distribution, while final recentering enforces the mean identity independently of optimization.

\subsection{Inference}

Structured perturbations of an observed ocean state establish the initial member conditions. At each day, the deterministic reference advances these conditions, and mean correction and coherent residual transport produce the evaluated ensemble through Eqs.~\eqref{eq:appendix_residual_composition} and \eqref{eq:recenter}. Reference outputs supply the next member conditions, while the shared context is reinitialized from the new condition. The coherent law therefore governs the conditional ensemble at each transition, with the reference determining progression in physical time. All reported \model{} ensemble metrics evaluate the transformed ensemble rather than the raw reference.

Simulation takes $(\mathbf F_t,\mathbf F_{t+1})$ from the test reanalysis, whereas prediction uses OneForecast atmospheric states. OneForecast is initialized at the 12:00 UTC state aligned with the ocean IC and advanced through four 6-hour steps per ocean day, supplying the current-day and next-day forcing pair. The ocean rollout protocol otherwise remains unchanged. The preliminary coupled-forcing experiment in Table~\ref{tab:forecasting_enso_2024} instead takes atmospheric states from an independently running two-way PnP-Corrector system. Here ``two-way'' identifies the forcing source: \model{} does not feed its ensemble back into that atmospheric model, and end-to-end probabilistic ocean--atmosphere coupling remains future work.


\section{Additional Results}
The additional results examine whether the gains in aggregate scores extend across depth, lead time, and individual ensemble members. They complement the main comparisons by separating the quality of the predictive center from the spatial structure and diversity of the conditional ensemble.

Figure~\ref{fig:visual_rmse_acc_crps_12vars} reports the 1--120-day simulation RMSE, ACC, and CRPS at representative subsurface levels (109.73 m, 222.48 m, and 453.94 m), with curves and shaded bands denoting the mean and standard deviation over the first 20 initial conditions starting from January 1, 2020 at one-day intervals. Due to ocean states typically vary smoothly in the vertical direction, adjacent depth levels often exhibit similar error patterns. We therefore report representative depths to provide a compact and interpretable assessment of the model's ability to capture three-dimensional ocean structure.

Figure~\ref{fig:fig_visual_mean_plus_fixed_members} compares ensemble-mean and individual-member simulation fields at the 120-day horizon for the first IC. Each model is represented by its ensemble mean and the same three fixed member indices (members 1, 10, and 20), rather than members selected by their agreement with the ground truth. The examples show how \model{} combines a spatially organized center with structured variation among individual members. These fields complement the member-averaged spectra and physical diagnostics in the main text: they illustrate the generated distribution beyond its mean, while the aggregate diagnostics assess its behavior across ICs and members.

Additional visual results are shown in Figures~\ref{fig_30.0-day}, \ref{fig_60.0-day}, \ref{fig_90.0-day}, \ref{fig_120.0-day}, covering 17 key ocean variables at lead times from 30 to 120 days, initialized from Jan.~1, 2020. For salinity, temperature, and SSH, whose variations are relatively small compared with their background values, we visualize anomalies to better highlight these subtle yet meaningful changes. In contrast, the original fields are directly shown for the zonal and meridional velocity components, U and V. For clarity, we provide several examples of the variable naming convention: SSSa denotes the sea surface salinity anomaly, while Sa109, Sa222, and Sa453 denote salinity anomalies at depths of 109.73~m, 222.48~m, and 453.94~m, respectively. Similarly, U0 denotes the sea surface zonal velocity, and U109, U222, and U453 denote the zonal ocean velocity at the corresponding depths. Collectively, these additional visualizations further demonstrate the effectiveness of the proposed \model{} in capturing multivariate ocean-state evolution over extended simulation horizons.

\begin{figure}[ht]
\centering
\includegraphics[width=1\linewidth]{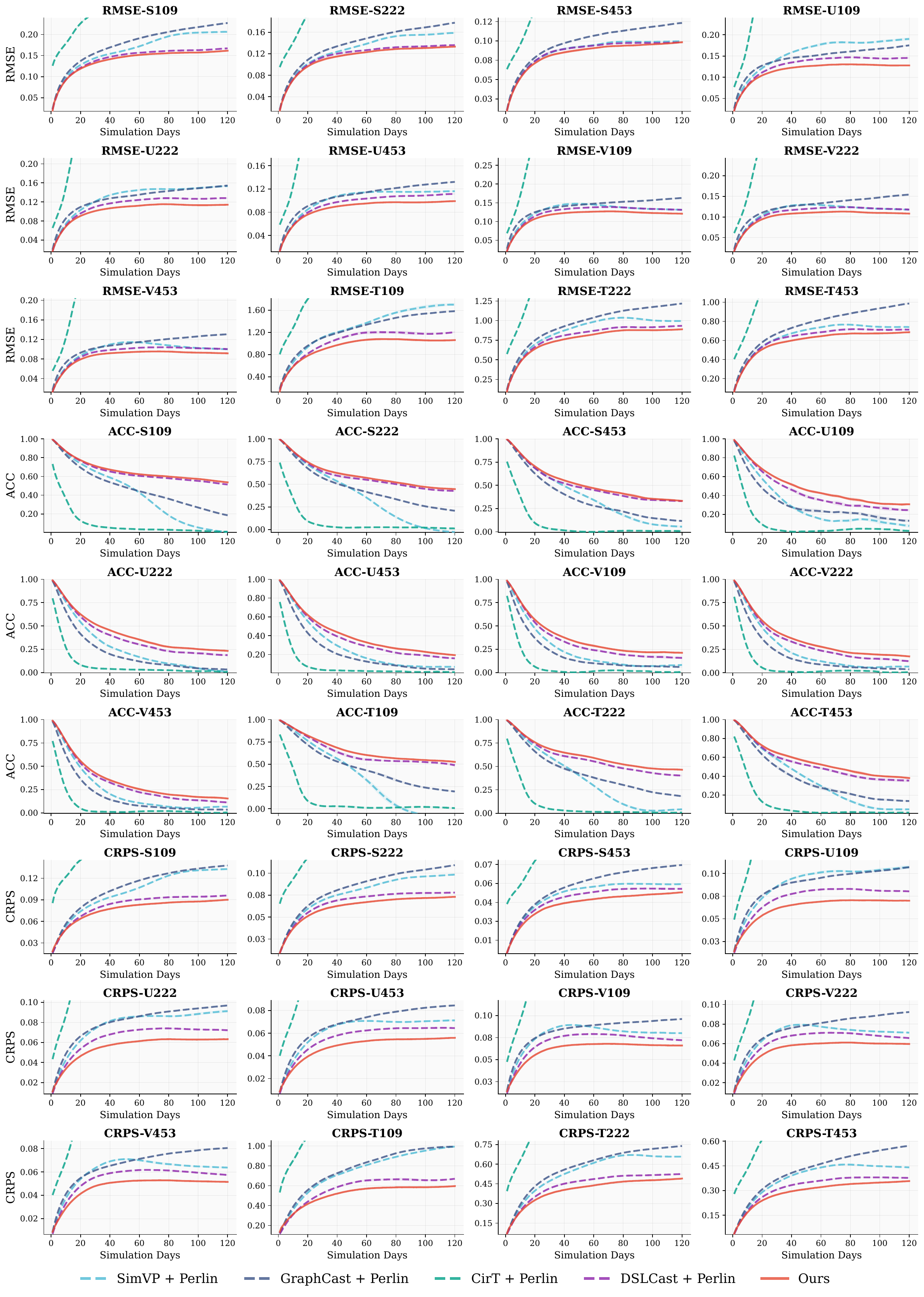}
\caption{Simulation RMSE, ACC, and CRPS over 1--120 days at depths of 109.73 m, 222.48 m, and 453.94 m. Lines and shaded bands show the mean and standard deviation across 20 ICs. Lower RMSE and CRPS and higher ACC are better.
}
\label{fig:visual_rmse_acc_crps_12vars}
\end{figure}

\begin{figure}[ht]
\centering
\includegraphics[width=1\linewidth]{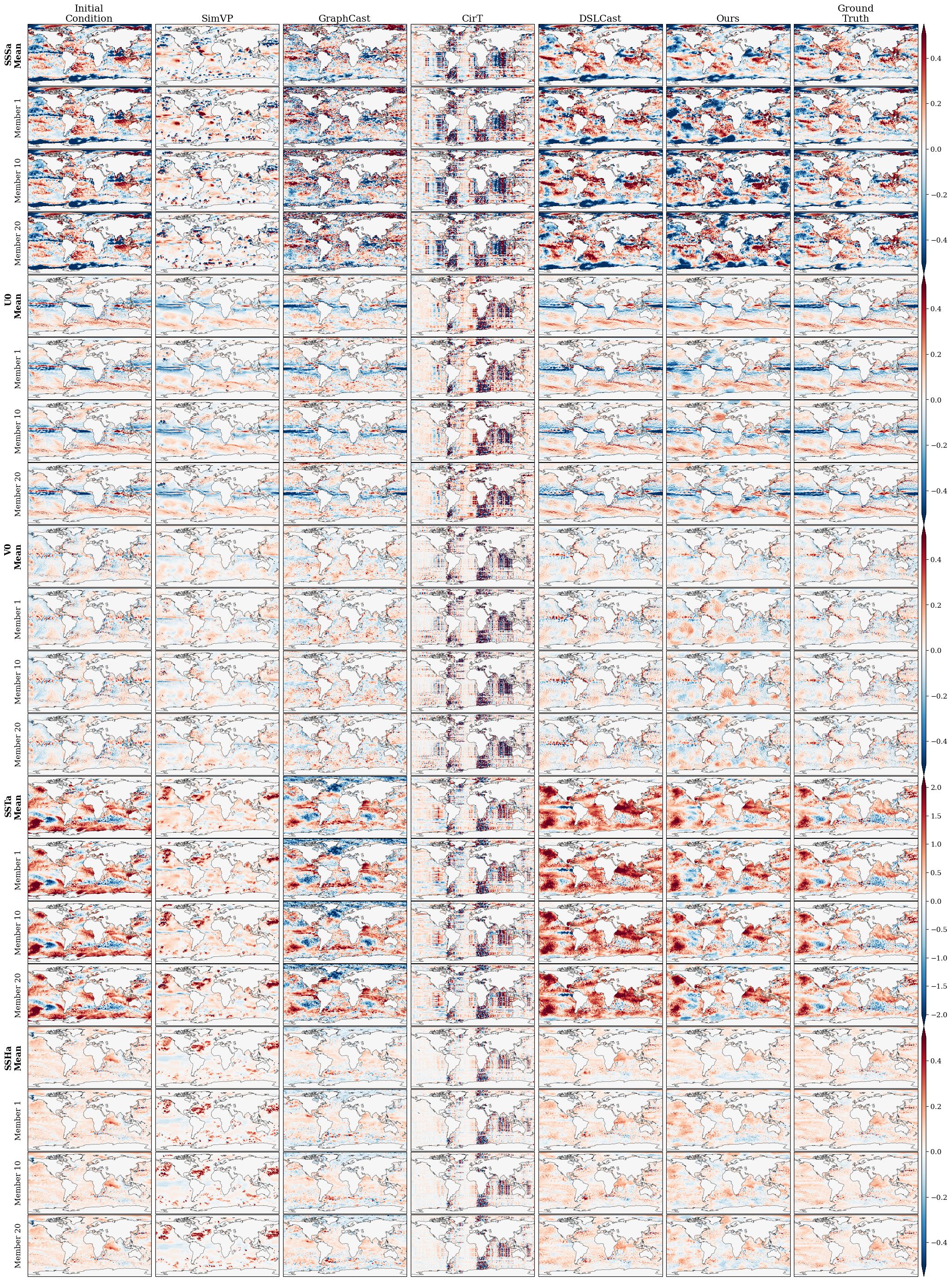}
\caption{120-day ensemble-mean and member-level simulation results at representative depth levels for the first IC. Individual members use fixed indices 1, 10, and 20 for every model.
}
\label{fig:fig_visual_mean_plus_fixed_members}
\end{figure}

\begin{figure}[ht]
\centering
\includegraphics[width=1\linewidth]{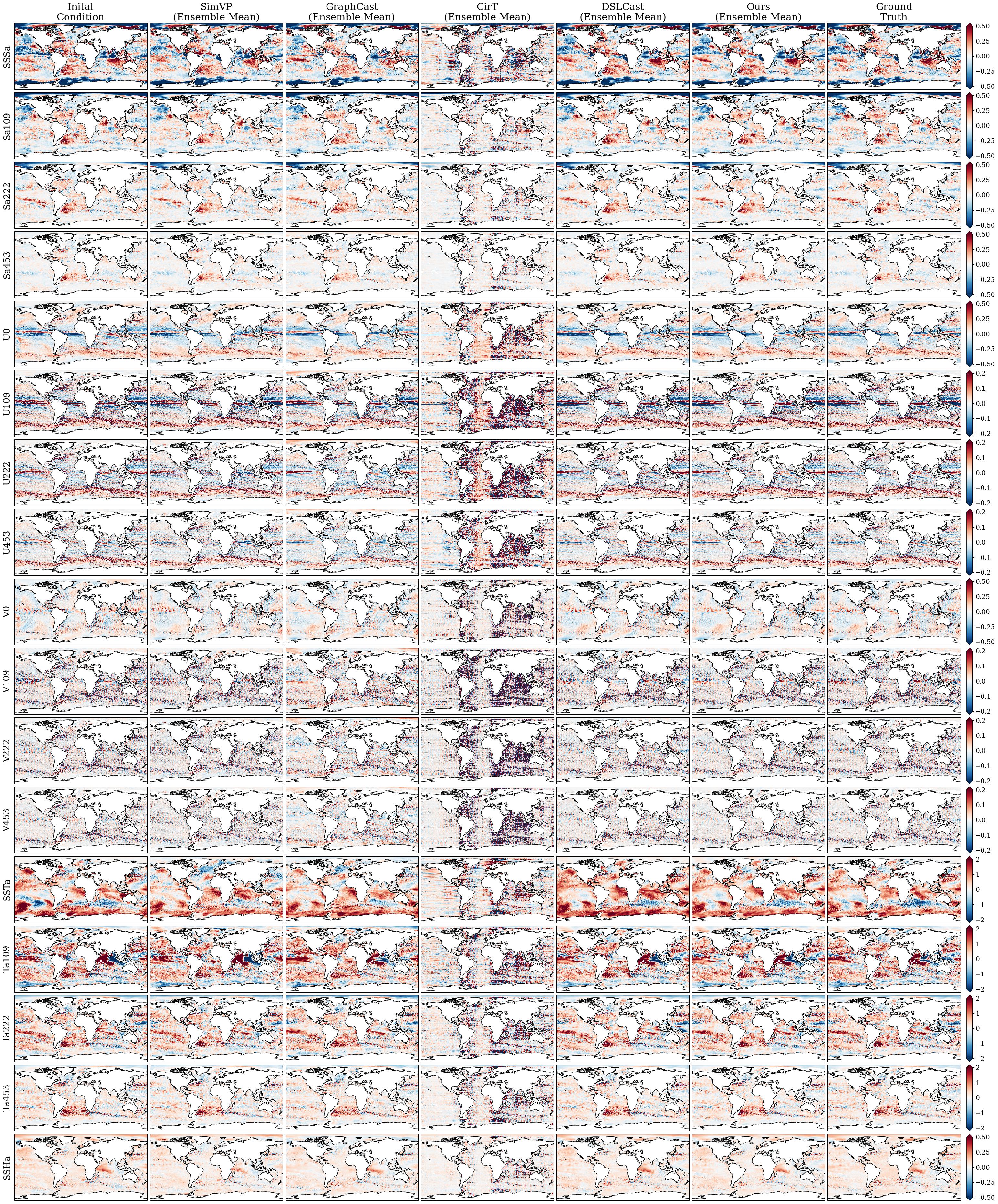}
\caption{30-day simulation results of different models at several representative depth levels.}
\label{fig_30.0-day}
\end{figure}

\begin{figure}[ht]
\centering
\includegraphics[width=1\linewidth]{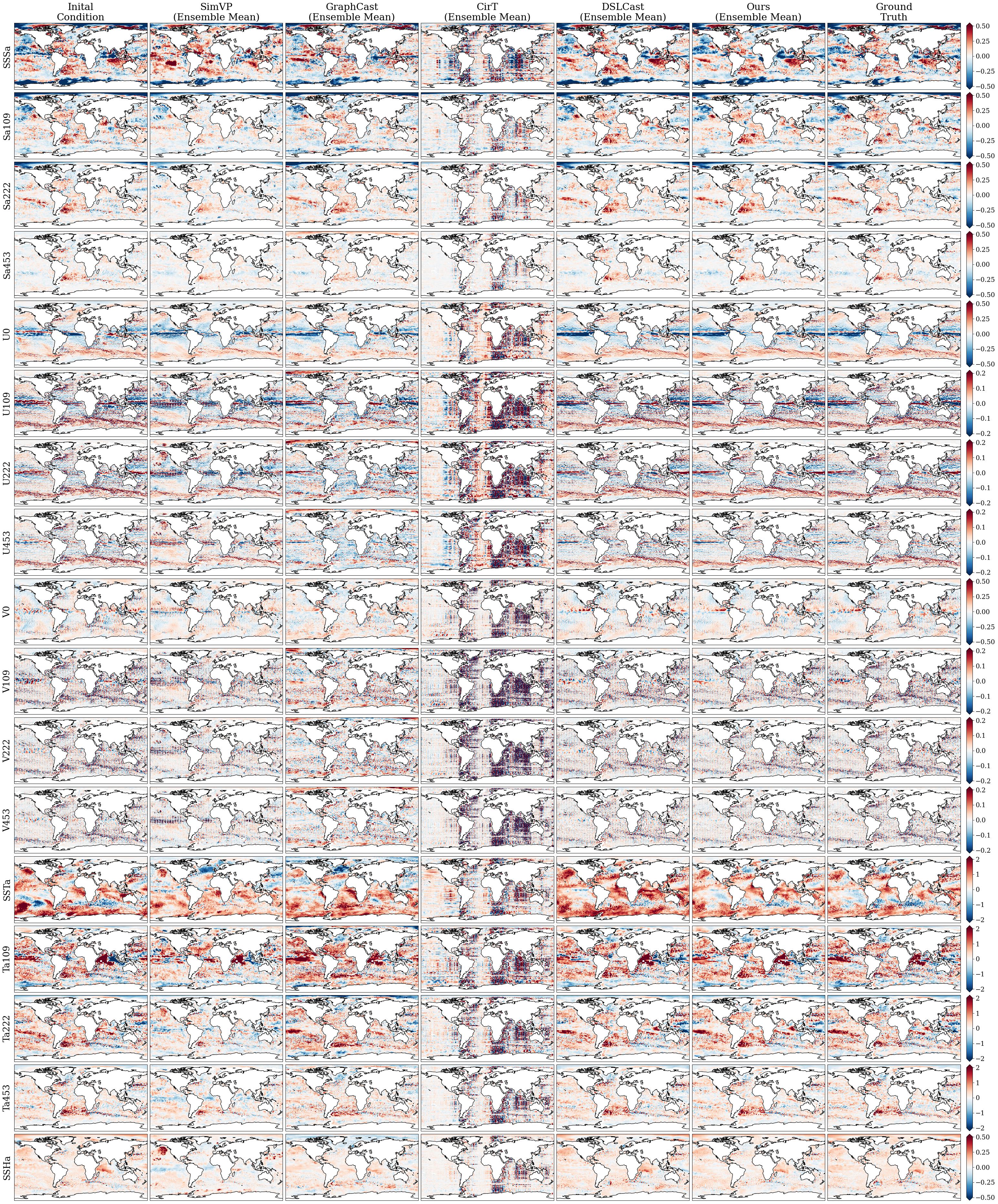}
\caption{60-day simulation results of different models at several representative depth levels.}
\label{fig_60.0-day}
\end{figure}

\begin{figure}[ht]
\centering
\includegraphics[width=1\linewidth]{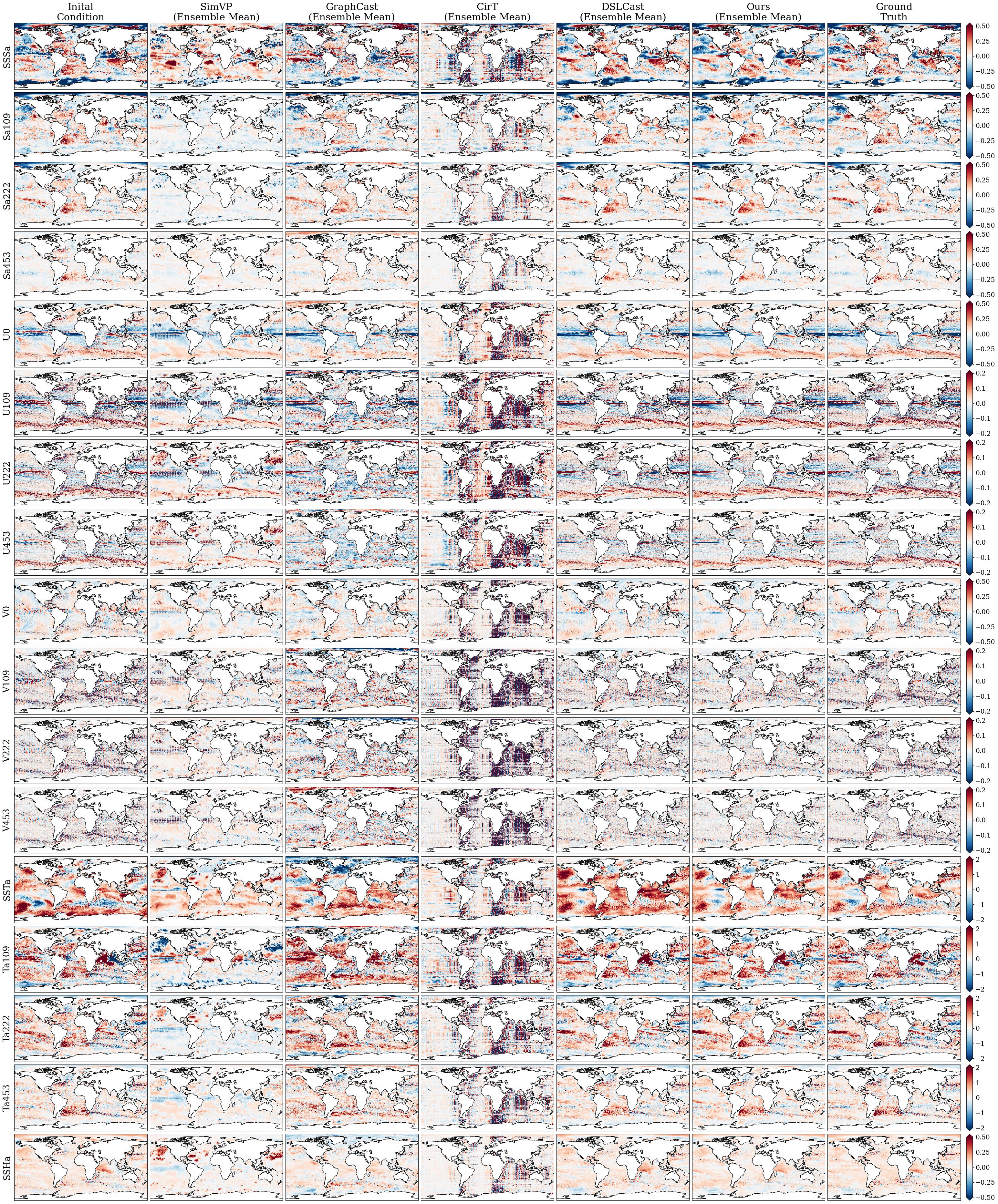}
\caption{90-day simulation results of different models at several representative depth levels.}
\label{fig_90.0-day}
\end{figure}

\begin{figure}[ht]
\centering
\includegraphics[width=1\linewidth]{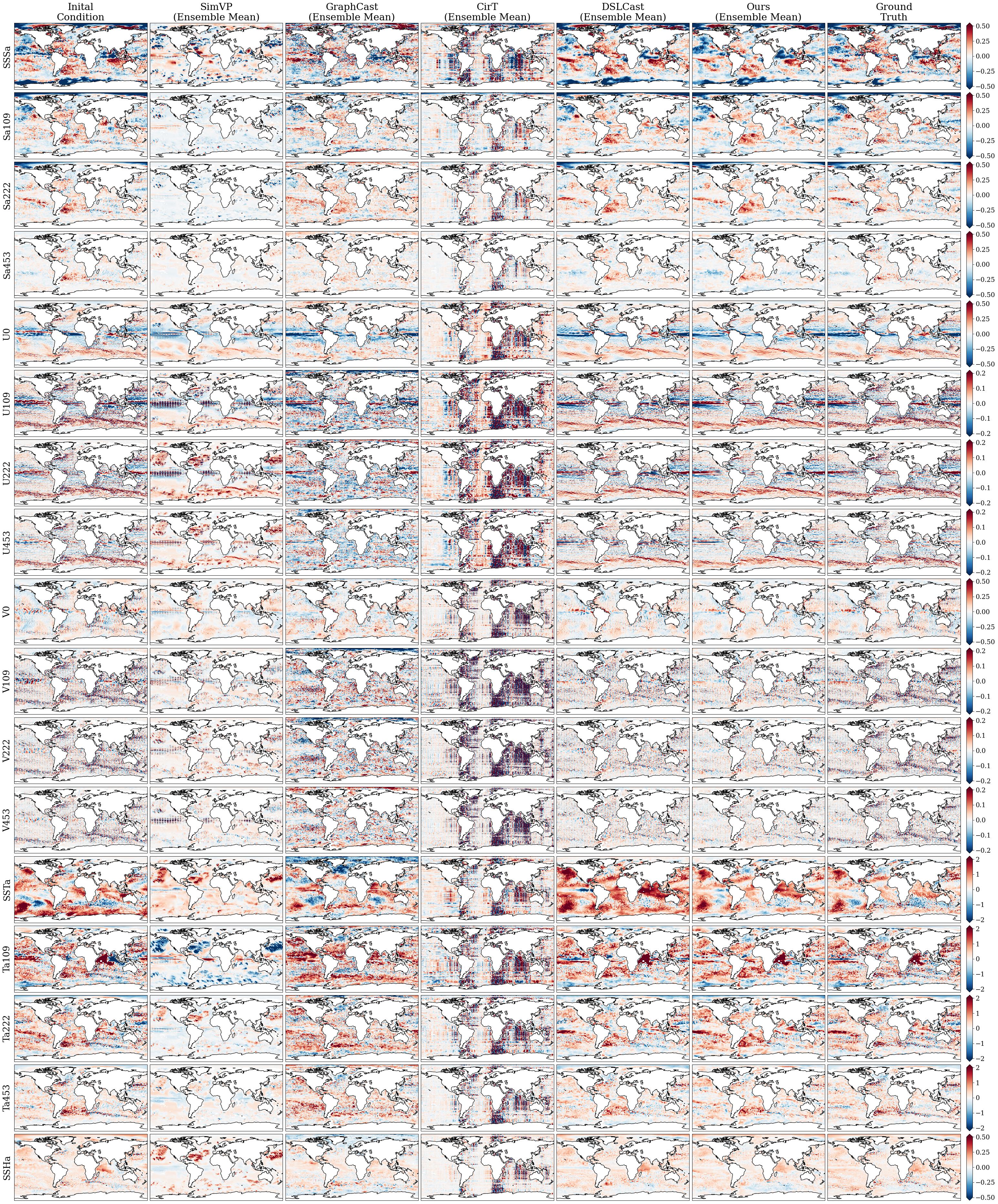}
\caption{120-day simulation results of different models at several representative depth levels.}
\label{fig_120.0-day}
\end{figure}

%% file: sec/2_Related_work.tex
\section{Related Work}

\paragraph{Deterministic ocean modeling.}
Numerical circulation models such as HYCOM~\citep{chassignet2007hycom} and NEMO~\citep{madec2015nemo}, together with state-estimation and reanalysis systems such as ECCO~\citep{forget2015ecco} and Copernicus global ocean products~\citep{lellouche2018recent}, provide the physical and observational foundations of ocean prediction. However, the integration remains computationally intensive, limiting the practical scale of ensemble generation. Data-driven models offer efficient alternatives with increasingly competitive deterministic skill. AI-GOMS~\citep{xiong2023ai} develops a reusable ocean modeling backbone; XiHe~\citep{wang2024xihe} and WenHai~\citep{cui2025forecasting} demonstrate eddy-resolving global prediction; LangYa~\citep{yang2024langya} addresses cross-spatiotemporal prediction. NeuralOM~\citep{gao2026neuralom} extends data-driven ocean modeling to subseasonal-to-seasonal simulation. The growing diversity of these models calls for consistent protocols to assess their relative performance. OceanForecastBench~\citep{jia2025oceanforecastbench} further supports systematic evaluation. These advances make ensemble generation more accessible, but accurate deterministic trajectories alone do not characterize the conditional uncertainty of future ocean states.

\paragraph{Probabilistic Earth-system modeling.}
Data-driven probabilistic models have found increasing application in Earth-system prediction, particularly in weather and precipitation forecasting~\citep{larsson2025crps,nai2025boosting,yu2024diffcast}. Examples include MetNet~\citep{sonderby2020metnet}, generative radar nowcasting~\citep{ravuri2021skilful}, diffusion-based weather ensembles with GenCast~\citep{price2023gencast}, and learned ensemble generation with FuXi-ENS~\citep{zhong2025fuxi}. Despite these advances, data-driven probabilistic ocean modeling remains comparatively underexplored. Conventional ocean ensembles~\citep{lorenz2017deterministic,leith1974theoretical,al2025data,wu2025spatiotemporal} represent uncertainty through initial-condition perturbations~\citep{toth1993ensemble,molteni1996ecmwf}, ensemble data assimilation~\citep{evensen1994sequential,houtekamer1998data}, or multiple models~\citep{buizza2005comparison,palmer2019ecmwf}, but require computationally expensive numerical integrations. Perturbing pretrained AI ocean models enables efficient ensemble generation and can improve ensemble-mean accuracy. However, independent rollouts lack joint distributional optimization and may misrepresent uncertainty in ocean dynamics. FuXi-ONS~\citep{huang2026data} learns state-dependent, spatially structured perturbations, but its memberwise objective leaves calibration implicit and its updates do not explicitly use ensemble context. \textit{In contrast, \model{} learns a coherent state-conditioned pathwise law that couples member residuals through the evolving ensemble context, jointly refining their conditional distribution while preserving a controlled mean anchor.}